\documentclass{article}

\usepackage{Mis/arxiv}

\usepackage[utf8]{inputenc} 
\usepackage[T1]{fontenc}    
\usepackage{hyperref}       
\usepackage{url}            
\usepackage{booktabs}       
\usepackage{amsfonts}       
\usepackage{nicefrac}       
\usepackage{microtype}      
\usepackage{lipsum}		
\usepackage{graphicx}
\usepackage{natbib}
\usepackage{doi}

\title{Towards Scalable and Robust Structured Bandits: \\ A Meta-Learning Framework
}

\date{} 					

\author{%
  Runzhe Wan\footnote{Equal contribution} \;\; Lin Ge$^{*}$ \;\; Rui Song\\
  Department of Statistics\\
  North Carolina State University\\
  \texttt{\{rwan, lge, rsong\}@ncsu.edu} \\
}

\renewcommand{\shorttitle}{Meta-learning Large-Scale Structured Bandits}

\hypersetup{
pdftitle={A template for the arxiv style},
pdfsubject={q-bio.NC, q-bio.QM},
pdfauthor={David S.~Hippocampus, Elias D.~Striatum},
pdfkeywords={First keyword, Second keyword, More},
}

\usepackage[utf8]{inputenc}
\usepackage{array}
\usepackage{amsmath}
\usepackage{amssymb}
\usepackage{bm}
\usepackage{natbib}
\usepackage{graphicx}
\usepackage{booktabs}
\usepackage{setspace}
\usepackage{textcomp}
\usepackage{enumitem}
\usepackage{datetime2}
\usepackage{multirow}
\usepackage[font=small,labelfont=bf]{caption}
\usepackage{todonotes}

\DeclareMathOperator*{\argmax}{arg\;max}
\usepackage{subcaption}

\usepackage{xcolor}
\usepackage{bbm}

\usepackage{lineno}
\usepackage[ruled,vlined]{algorithm2e}

\newcommand{\vx}{\mathbf{x}}

\newcommand{\vbeta}{\boldsymbol{\beta}}

\newcommand{\vthe}{\boldsymbol{\theta}}
\newcommand{\veta}{\boldsymbol{\eta}}

\newcommand{\mH}{\mathcal{H}}

\newcommand{\vY}{\boldsymbol{Y}}

\newcommand{\vgamma}{\boldsymbol{\gamma}}

\newcommand{\TO}{\tilde{O}}

\newcommand{\mA}{\mathcal{A}}

\newcommand{\I}{\mathbb{I}}

\newcommand{\normal}{\mathcal{N}}

\usepackage{soul}
\usepackage{xcolor}

\newcommand{\Mean}{{\mathbb{E}}}

\newcommand{\Cov}{\boldsymbol{\Sigma}}

\newcommand{\prob}{\mathbb{P}}

\newcommand{\vphi}{\boldsymbol{\phi}}
\newcommand{\vPhi}{\boldsymbol{\Phi}}

\newcommand{\vZ}{\boldsymbol{Z}}

\newcommand{\vr}{\boldsymbol{r}}
\newcommand{\vone}{\boldsymbol{1}}
\newcommand{\vzero}{\boldsymbol{0}}
\newcommand{\vI}{\boldsymbol{I}}
\newcommand{\vmu}{\boldsymbol{\mu}}
\newcommand{\vB}{\boldsymbol{B}}

\usepackage{amsthm}
\theoremstyle{plain}

\newcommand{\name}[1]{\texttt{#1}}

\newtheorem{thm}{Theorem}
\newtheorem{theorem}{Theorem}

\newtheorem{lemma}[thm]{Lemma}

\newtheorem{remark}{Remark}

\newlength{\leftstackrelawd}
\newlength{\leftstackrelbwd}
\def\eqstack#1#2{\settowidth{\leftstackrelawd}%
{${{}^{#1}}$}\settowidth{\leftstackrelbwd}{$#2$}%
\addtolength{\leftstackrelawd}{-\leftstackrelbwd}%
\leavevmode\ifthenelse{\lengthtest{\leftstackrelawd>0pt}}%
{\kern-.5\leftstackrelawd}{}\mathrel{\mathop{#2}\limits^{#1}}}

\usepackage[ruled,vlined]{algorithm2e}

\usepackage{natbib}
\setcitestyle{authoryear,round,citesep={;},aysep={,},yysep={;}}

\definecolor{mydarkblue}{rgb}{0,0.08,0.45}
\hypersetup{ %
colorlinks=true,
linkcolor=mydarkblue,
citecolor=mydarkblue,
filecolor=mydarkblue,
urlcolor=mydarkblue,
}
\renewcommand{\cite}[1]{\citep{#1}}

\begin{document}
\maketitle

\let\svthefootnote\thefootnote
\newcommand\freefootnote[1]{%
  \let\thefootnote\relax%
  \footnotetext{#1}%
  \let\thefootnote\svthefootnote%
}

\freefootnote{*: Equal contribution.}


\begin{abstract}
Online learning in large-scale structured bandits is known to be challenging due to the curse of dimensionality. 
In this paper, we propose a unified meta-learning framework for a general class of structured bandit problems where the  parameter space can be factorized to item-level. 
The novel bandit algorithm is general to be applied to many popular problems,
scalable to the huge parameter and action spaces, 
and robust to the specification of the generalization model. 
At the core of this framework is a Bayesian hierarchical model that allows information sharing among items via their features, upon which we design a meta Thompson sampling algorithm. 
Three representative examples are discussed thoroughly. 
Both theoretical analysis and numerical results support the usefulness of the proposed method. 
\end{abstract}
\vspace{-.3cm}
\section{Introduction}
The bandit problem has received increasing attention and has been widely applied to areas such as clinical trials \citep{durand2018contextual}, finance \citep{shen2015portfolio}, recommendation systems \citep{zhou2017large}, among others. 
However, many real-world applications typically have a large number of unknown parameters, a huge action space, and a complex reward distribution specified by domain models. 
For instance, in online learning to rank, the agent typically needs to choose a slate from more than thousands of related items \cite{li2016contextual, zong2016cascading}, and online advertising on major websites is usually viewed as a bipartite matching problem with millions of users and items \cite{wen2015efficient}. 
 \textit{How to efficiently explore and learn in a complex and large-scale structured bandit problem} is known to be challenging \citep{oh2019thompson, wen2015efficient, li2016contextual,  zong2016cascading}, which makes standard bandit algorithms intractable and impedes the deployment of bandits in many real systems. 

In this paper, we focus on a class of structured bandit problems where the parameter space can be factorized and each parameter is related to one item. 
Here, an item can be a product, a web page, a movie, etc., depending on the application. 
Such a problem is very general and includes many popular bandit problems as special cases, including 
dynamic assortment optimization \citep{agrawal2017thompson, agrawal2019mnl}, online learning to rank \citep{kveton2015cascading, cheung2019thompson}, online combinatorial optimization  \citep{chen2013combinatorial, wang2018thompson}, rank-$1$ bandits \citep{katariya2017stochastic}, etc. 
In addition, we consider incorporating items' feature information, which is commonly available.

To address the scalability issue for this wide class of problems, we propose a meta-learning framework. 
Specifically, we first build a Bayesian hierarchical model to allow information sharing among items via their features, upon which we then design a Thompson sampling \citep[TS,][]{russo2017tutorial}-type  algorithm. 
The hierarchical model provides us a principled way to construct a feature-based informative prior for each item, which guides the exploration of TS. 
As such, our method can be viewed as \textit{learning how to learn} efficiently for each item, and hence for the whole problem.


\textbf{Contribution.} 
Our contributions are multi-fold. 
First, to address the challenges in large-scale structured bandits, 
we propose a unified meta-learning framework with a TS-type algorithm, 
named Meta Thompson Sampling for Structured bandits (\name{MTSS}). 
To the best of our knowledge, this is the \textit{first} work using meta-learning to solve a single complicated bandit problem. 
The framework is general to subsume a wide class of practical problems, scalable to large systems, and robust to the specification of the generalization model. 
Notably, in contrast to existing feature-determined approaches which require the item-specific parameters to be predicted \textit{perfectly} by the features, our framework is flexible to allow utilizing feature information while avoiding the bias of an over-stylized model. 
Besides, when combined with the proposed offline-training-online-deployment schedule, \name{MTSS} yields low system latency and hence is particularly suitable for large-scale systems. 
The framework is attractive for cold-start problems as well. 

Second, to illustrate our framework, we discuss three concrete examples thoroughly, which are strongly related to the literature. 
Our framework provides a practical solution to these application domains, including ranking, combinatorial optimization, and assortment optimization. 

Third, we provide a general information-theoretic regret bound for \name{MTSS}, which is easy to adapt to a specific problem. The bound decomposes into two parts: the price of learning the generalization function and the regret even with the generalization function known in advance. 
As an example, we derive the regret bound under semi-bandits and show that the regret of \name{MTSS} due to not knowing the generalization function is asymptotically negligible and does not grow with the number of items, unlike existing approaches. 
This highlights the benefits of meta-learning.

Finally, in three applications,
we compare our approach with existing methods using extensive experiments on both synthetic and real datasets. 
The results show that the proposed framework can learn efficiently in large problems, is computationally attractive, yields robustness to model misspecification, and is useful for cold-start problems.













\vspace{-.2cm}
\section{Setup and Existing Approaches}
We consider the following popular and general class of online decision problems \citep{russo2017tutorial}: 
\begin{equation}
    \begin{split}
    \vY_t &\sim f(\vY_t |A_t, \vthe),\\
    R_t &= f_r(\vY_t; \veta). 
    \end{split}
    \label{eqn:main_raw_model}
\end{equation}
Here, for $t = 1, \dots, T$, the agent will sequentially choose action $A_t$ from the action space $\mA$ and then receive corresponding stochastic observations $\vY_t$, which determines the  reward $R_t$ through a deterministic function $f_r$ parameterized by some \textit{known} parameters $\veta$. 
The observation $\vY_t$ is generated following a domain model $f$, which is parameterized by some \textit{unknown} parameters $\vthe$. 
In many real problems, $f$ is typically a complex distribution involving nonlinear functions, $\vthe$ is high-dimensional, and the action space $\mathcal{A}$ is huge. 

Denote $r(a, \vthe) = \Mean(R_t|A_t = a, \vthe)$ as the expected reward of taking action $a$ in a problem instance with parameter $\vthe$. 
One common metric is the cumulative regret
\begin{align*}
    R(T, \vthe) = 
   \sum_{t=1}^T \big[ 
   \max_{a \in \mA} r(a, \vthe) - r(A_t, \vthe)
   \big]. 
\end{align*}
In many applications, the structured bandit problem consists of $N$ items, and the unknown parameter $\vthe$, admittedly being high-dimensional, can be factorized over these items as $\vthe = (\theta_1, \dots, \theta_N)^T$,  where $\theta_i$ is the parameter related to the $i$th item. 
This problem setting subsumes many popular bandit problems, such as 
dynamic assortment optimization where the agent needs to recommend a 
subset of items, 
online learning to rank where the agent needs to generate a ranked slate, 
combinatorial semi-bandits which have numerous applications including online advertisement and optimal network routing, and many others. 
In this paper, we will focus on this class of structured bandit problems, and will discuss three representative examples in Section \ref{sec:example}. 



\textbf{Existing approaches.}
The existing works typically study one specific instance in this class, and these methods can be categorized as either \textit{feature-agnostic} or \textit{feature-determined}. 
Feature-agnostic approaches \citep{chen2013combinatorial, wang2018thompson, kveton2015cascading, cheung2019thompson, agrawal2017thompson, agrawal2019mnl} do not utilize side information such as  features and learn each $\theta_i$ independently. 
Most of them adapt either the upper-confidence bound (UCB) or TS framework. 
In these existing works, the regret bounds will scale quickly with the number of items $N$, which could be prohibitive in many modern applications. 
Therefore, feature-agnostic approaches are known to be not scalable and even show a (nearly) linear regret in some experiments \citep{wen2015efficient, zong2016cascading, ou2018multinomial, agrawal2020tractable}. 

To address the scalability issue, feature-determined approaches \citep{wen2015efficient, zong2016cascading, ou2018multinomial, agrawal2020tractable} notice that commonly each item $i$ has some features $\vx_i$. 
They assume that there is a \textit{deterministic} function $g$ parameterized by $\vgamma$ such that $\theta_i {=} g(\vx_i; \vgamma)$ with no error. 
When this restricted generalization model assumption holds, 
the regret bound for feature-determined approaches can be independent of $N$ but depend on the number of features $d$ instead. 
Such an argument makes the feature-determined methods theoretically attractive when $d$ is relatively small compared with $N$. 

However, feature-determined approaches have two major drawbacks. 
First, as usual, algorithms designed with a restrictive model assumption are brittle. 
No matter how informative $\vx_i$ is and how complex $g$ is, it is typically challenging to ensure $\theta_i \equiv g(\vx_i)$ without stochastic errors. 
This misspecification is exacerbated when almost all existing works assume $g$ as linear, given the computational challenge. 
In the bandit setting, when $\theta_i = g(\vx_i)$ does not hold,  the regret is easy to scale linearly with $T$. 
Our experiments will further illustrate this observation. 
Second, feature-determined approaches are typically computationally demanding for online updating, which may cause system latency issues in online deployment. 
This is due to that we add an additional generalization model to the already complex structured bandit model and also have to update the full model as a whole. 







\section{General Framework}\label{sec:framework} 
To combine the merits of both approaches and hence enable scalable and robust bandit learning, we propose a meta-learning framework: we first build a Bayesian hierarchical model that enables information sharing among items via their features, and then design a TS algorithm that learns the information-sharing structure while minimizing the cumulative regrets. 
In this section, we will focus on the general framework, with examples given in Section \ref{sec:example}. 
For any positive integer $M$, we denote the set $\{1, \dots, M\}$ by $[M]$. 

\subsection{Feature-based hierarchical model}\label{sec:Model}
With a large number of items, we adopt the \textit{meta-learning} viewpoint \citep{vilalta2002perspective}, by regarding the items $\{(\vx_i, \theta_i)\}$ as sampled from a joint distribution. 
To allow information sharing while mitigating the issue from a deterministic generalization model, we model the item-specific parameter $\theta_i$ as sampled from a certain distribution $g(\theta_i|\vx_i, \vgamma)$ instead of being entirely determined by $\vx_i$. 
Here, $g$ is a model parameterized by an \textit{unknown} vector $\vgamma$, which we will instantiate shortly with examples. 
Therefore, 
combining with the base model \eqref{eqn:main_raw_model}, 
we consider the following hierarchical model: 

\begin{equation}\label{eqn:general_hierachical}
  \begin{alignedat}{2}
&\text{(Prior)} \quad
\quad\quad\quad\quad\quad\quad\quad\quad\quad
\vgamma &&\sim Q(\vgamma),\\
&\text{(Generalization function)} \;
\;    \theta_i| \vx_i, \vgamma  &&\sim g(\theta_i|\vx_i, \vgamma), \forall i \in [N],\\ 
&\text{(Observations)} \quad\quad\quad\quad\quad\quad\;
\;    \boldsymbol{Y}_t &&\sim f(\boldsymbol{Y}_t|A_t, \vthe),\\
&\text{(Reward)} \quad\quad\quad\quad\quad\quad\quad\quad\;
\;   R_t &&= f_r(\boldsymbol{Y}_t ; \veta), 
      \end{alignedat}
\end{equation}

where $Q(\vgamma)$ is the prior distribution for $\vgamma$. 
Intuitively, as such, we can share information across items via $g$ to infer any $\theta_i$ and speed up learning, while we can also utilize the observations $\{\boldsymbol{Y}_t\}$ to estimate $\theta_i$ in an unbiased way via $f$. 
Compared with the two existing approaches, the main difference can be concisely summarized in Table \ref{tab:comparison}. 


\setlength{\tabcolsep}{2pt}
\begin{table}[t]
\centering
\caption{Comparison of key model assumptions.}
\label{tab:comparison}
\begin{tabular}{ccc}
\toprule
{Feature-agnostic} & {Feature-determined} & Feature-guided (ours)   \\
\midrule
$\theta_i \sim \prob(\theta)$  & $\theta_i = g(\vx_i; \vgamma)$ &  $\theta_i \sim g(\theta_i|\vx_i,  \vgamma)$  \\
\bottomrule
\end{tabular}
\end{table}

Finally, from the meta-learning point of view, 
it is more common to consider the Bayes regret \cite{kveton2021meta}: 
\begin{align*}
     BR(T) = \Mean_{
     \vgamma \sim Q(\vgamma), 
     \vx_i \sim \prob(\vx_i), \theta_i \sim g(\theta_i|\vx_i, \vgamma)}  R(T, \vthe), 
\end{align*}
where the expectation is additionally taken over the item distribution and the prior distribution $Q(\vgamma)$. 

\subsection{Meta Thompson sampling with feature-guided exploration}\label{sec:main_alg}
On the foundation of the hierarchical model \eqref{eqn:general_hierachical}, we propose our bandit algorithm 
in Algorithm \ref{alg:main_TS}, which is a natural and general TS-type algorithm. 
TS is one of the most popular bandit algorithm frameworks  \citep{russo2017tutorial, lattimore2020bandit}, with superior numerical and theoretical performance. 
As a Bayesian algorithm, TS samples the action at each round from the posterior distribution of the optimal action.


For a given structured bandit problem, once the generalization model $g$ and the prior are specified, 
the remaining steps to adapt Algorithm \ref{alg:main_TS} 
are updating the posterior (step 1-4) and solving the optimization problem (step 5). 
This optimization step is problem-dependent, and can typically be solved efficiently via existing methods in the corresponding structured bandit literature. 

\begin{algorithm}[!t]
\SetKwData{Left}{left}\SetKwData{This}{this}\SetKwData{Up}{up}
\SetKwFunction{Union}{Union}\SetKwFunction{FindCompress}{FindCompress}
\SetKwInOut{Input}{Input}\SetKwInOut{Output}{output}
\SetAlgoLined
\Input{
Prior $Q(\vgamma)$ and known parameters of the hierarchical model
}
Set $\mH_{1} = \{\}$

\While{$t<T$}{ 

1. Update the posterior for $\vgamma$ as $\prob(\vgamma | \mH_{t})$, according to the  hierarchical model \eqref{eqn:general_hierachical}\\

2. Sample $\Tilde{\vgamma} \sim \prob(\vgamma | \mH_{t})$ \\ 

3. Update the posterior for $\vthe$ as $\prob(\vthe | \mH_{t}, \Tilde{\vgamma})$, according to model \eqref{eqn:main_raw_model} with $g(\theta_i \mid \vx_i, \Tilde{\vgamma})$ as the prior for each $\theta_i$
\\

4. Sample $\Tilde{\vthe} \sim \prob(\vthe | \mH_{t}, \Tilde{\vgamma})$\\

5. Take the greedy action $A_t$ w.r.t. $\tilde{\vthe}$ as 
$A_t = \argmax_{a \in \mA} \Mean(R_t \mid a, \tilde{\vthe})$\\ 
6. Receive reward $R_t$ and update the dataset as $\mH_{t+1} \leftarrow \mH_{t} \cup \{(A_{t}, R_{t})\}$

}
\caption{
\name{MTSS}: Meta Thompson Sampling for Structured bandits 
}\label{alg:main_TS}
\end{algorithm}

We will discuss the posterior updating step in depth in Section \ref{sec:offline_training}. 
Before we proceed, we remark that step 1-4 of Algorithm 
\ref{alg:main_TS} can actually be written concisely as sampling $\Tilde{\vthe}$ from its posterior based on the hierarchical model \eqref{eqn:general_hierachical}, which can be seen from the relationship
\begin{align*}
    \prob(\vthe \mid \mH) = \int_{\vgamma} \prob(\vthe \mid \vgamma, \mH)  \prob(\vgamma \mid \mH) d \vgamma. 
\end{align*}
Therefore, Algorithm 
\ref{alg:main_TS} can be regarded as a TS-type algorithm. 
We split the posterior updating process into steps 1-4 for two major reasons. 
First, in many cases, it is computationally more efficient to update the posteriors of $\vgamma$ and $\vthe$ separately, as will be discussed in the next section. 
Second, this decomposition provides a nice insight that our framework actually constructs a feature-based informative prior $g(\theta_i | \vx_i, \Tilde{\vgamma})$ for each $\theta_i$ to guide the feature-agnostic TS algorithm, and the prior is obtained by pooling information across items via their features using the hierarchical model. 
As such, our approach is an instance of meta-learning \citep{vilalta2002perspective}, 
and hence we refer to Algorithm \ref{alg:main_TS} as Meta Thompson Sampling for Structured bandits (\name{MTSS}).

\begin{remark}
The proposed framework is particularly useful for cold-start problems, where new items will be frequently introduced. 
Without any historical interaction data, it is important to construct an informative prior for a new item based on its features, which provides an initial expectation and guides the exploration. 
\end{remark}

\subsection{Posterior updating and offline-training-online-deployment}\label{sec:offline_training}
In Algorithm \ref{alg:main_TS}, the posterior updating step can be computed either explicitly when the problem structure permits (see e.g., Section \ref{sec:semi-bandit}), or via approximate posterior inference algorithms, such as Gibbs sampler \cite{johnson2010gibbs} or variational inference \cite{blei2017variational}.

We note that the base model \eqref{eqn:main_raw_model} typically yields a nice conjugate structure for $\vthe$ (e.g., in all three examples in Section \ref{sec:example}), and approximate posterior inference can be applied to $\vgamma$ alone in these cases. 
Approximate posterior inference is widely applied to TS \cite{yu2020graphical,wan2021metadata}, and is particularly appropriate in this case due to two reasons: 
(i) the posterior of $\vgamma$ is only used to construct a prior for the base model \eqref{eqn:main_raw_model}, and hence its error is of less concern, as related feature-agnostic TS algorithms typically enjoy prior-independent or instance-independent sublinear regret bounds \citep{wang2018thompson, perrault2020statistical, zhong2021thompson}; 
(ii) when computing the posterior of $\vgamma$, many approximate inference algorithms can benefit from the nice hierarchical structure and hence be efficient. For example, with Gibbs sampler, the algorithm will alternate between the posterior of $\vthe$, which yields a conjugate form, and that of $\vgamma$, which involves a Bayesian regression. 
Both parts can be solved efficiently. 

To facilitate computationally efficient deployment, we further propose an offline-training-online-deployment variant, 
where we only sample a new $\Tilde{\vgamma}$ at a certain time point $t \in \mathcal{T}$ instead of at every time point. 
For example, $\mathcal{T}$ can be $\{2^l : l = 1, 2, \dots\}$ or some trigger time every week. 
In other words, we will re-train the generalization model $g(\theta;\vx,\vgamma)$ \textit{offline} in a batch mode, 
and during the \textit{online} deployment, we only need to utilize the priors $\{g(\theta_i | \vx_i, \Tilde{\vgamma})\}$. 
As such, during the online deployment phase, our algorithm requires \textit{zero} additional computational cost compared to feature-agnostic TS. 
Therefore, \name{MTSS} in general yields low latency and hence is particularly suitable for deployments in large-scale systems. 
Besides, a powerful generalization function such as a Gaussian process or a Bayesian neural network also becomes feasible. 
This is a highly practical algorithm, and our numerical results further support its good performance. 
Finally, it can also be viewed as an empirical Bayes approach \citep{maritz2018empirical}. 

\section{Examples}\label{sec:example}
In this section, we illustrate our framework with three representative examples. 
For each of them, we will first write its feature-agnostic form in notations of model \eqref{eqn:main_raw_model}, 
and then discuss its applications and the optimization problem, 
and next instantiate model \eqref{eqn:general_hierachical} with an \textit{example} choice of $g$, 
and finally discuss the corresponding posterior computation to instantiate \name{MTSS}. 
Denote the cardinality of set $A$ by $|A|$. 
\subsection{Cascading bandits for online learning to rank}\label{sec:cascading}
The cascading model is popular in learning to rank \citep{chuklin2015click} to characterize how a user interacts with an ordered list of $K$ items. 
Its bandit version has attracted much attention recently, and both feature-agnostic \citep{kveton2015cascading, cheung2019thompson} and feature-determined approaches \citep{zong2016cascading} have been discussed.

In this model, $\mA$ contains all the subsets of $[N]$ with length $K$, $A_t = (a_t^1, \dots, a_t^K) \in \mA$ is a sorted list of items being displayed, 
$\vY_t$ is an indicator vector with the $a$th entry equal to $1$ when the $a$th displayed item is clicked, 
and $R_t$ is the reward with $f_r(\vY_t) \equiv \sum_{k \in [K]} Y_{k,t}  \in \{0,1\}$, where $Y_{k,t}$ is the $k$th entry of $\vY_t$. 
The model is intuitive and widely applied: 
the user will exam the $K$ displayed items from top to bottom, 
and stop to click one item once she is attracted (or leave if none of them is attractive). 
Let $I_t$ be the index of the chosen item if exists, and otherwise let $I_t = K$. 
To formally define the model $f$, it is useful to introduce a latent binary variable $E_{k,t}$ to indicate if the $k$th displayed item is examined by the $t$th user, 
and a latent variable $W_{k,t}$ to indicate if the $k$th displayed item is attractive to the $t$th user. 
Therefore, the value of $W_{k,t}$ is only visible when $k \le I_t$.
Let $\theta_i$ be the attractiveness of the item $i$. 
The key probabilistic assumption is that $W_{k, t} \sim Bernoulli(\theta_{a^k_t}), \forall k \in [K]$. 
When $\vthe$ is known, the optimal action can be shown as any permutation of the top $K$ items with the highest attractiveness factors. 

To characterize the relationship between items using their features, 
one example choice of $g$ is the popular Beta-Bernoulli logistic model \citep{forcina1988regression, wan2021metadata}, where $\theta_i {\sim} Beta(logistic(\vx_i^T \vgamma), \psi)$ for some known parameter $\psi$. 
Hereinafter, we adopt the mean-precision parameterization of the Beta distribution, with $logistic(\vx_i^T \vgamma)$ being the mean and $\psi$ being the precision parameter. 
Therefore, our model is 
\begin{equation}\label{eqn:hierachical_model_cascading}
    \begin{split}
     \theta_i &\sim Beta(logistic(\vx_i^T \vgamma), \psi), \forall i \in [N],\\
    W_{k, t} &\sim Bernoulli(\theta_{a^k_t}), \forall k \in [K], \\
    Y_{k,t} &= W_{k,t} E_{k,t}, \forall k \in [K],\\
    E_{k,t} &= (1-Y_{k-1}) E_{k-1,t}, \forall k \in [K],\\
    R_t &= \sum_{k \in [K]} Y_{k,t}, 
    \end{split}
\end{equation}
with $E_{1,t} \equiv 1$. 
With a given $\vgamma$, the posterior of $\vthe$  enjoys the  Beta-Bernoulli conjugate relationship and hence can be updated explicitly and efficiently. 
The prior $Q(\vgamma)$ can be chosen as many appropriate distributions such as Gaussian. 
To update the posterior of $\vgamma$, we can apply approximate inference as discussed in Section \ref{sec:offline_training}. 
Many other learning to rank models, such as the position-based model \citep{chuklin2015click}, can be formulated and solved similarly. 


\subsection{Combinatorial semi-bandits for online combinatorial optimization}
\label{sec:semi-bandit}

Online combinatorial optimization has a wide range of  applications \citep{sankararaman2016semi}, including maximum weighted matching, ad allocation, and news page optimization, to name a few. 
It is also common that all chosen items will  generate a separate observation, known as the  semi-bandit problem \citep{chen2013combinatorial}. 
Both the feature-agnostic \citep{chen2013combinatorial, wang2018thompson} and the feature-determined \citep{wen2015efficient} approaches have been studied. 
Formally, in a combinatorial semi-bandit, the feasible set $\mA \subseteq \{A \subseteq [N] : |A|\leq K\}$ consists of subsets that satisfy the size constraint and other application-specific constraints. 
The agent will sequentially choose a subset $A_t$ from $\mA$, and then receive a separate reward $Y_{i, t}$ for each chosen item $i \in A_t$, with the overall reward defined as $R_t = \sum_{i \in A_t} Y_{i,t}$. 
When the mean reward of each item is known, the optimal action can be obtained from a combinatorial optimization problem, which can be efficiently solved in most real applications considered in the literature by corresponding combinatorial optimization algorithms \citep{chen2013combinatorial}. 

In this section, as an example, we focus on the popular case where $Y_{i, t}$ follows a Guassian distribution 
and consider using a linear mixed model (LMM) as the generalization model. 
Specifically, the full model is as follows: 
\begin{equation}\label{eqn:LMM}
    \begin{split}
     \theta_i &\sim \normal(\vx_i^T \vgamma, \sigma_1^2), \forall i \in [N],\\
    Y_{i, t} &\sim \normal(\theta_i, \sigma_2^2), \forall i \in A_{t},\\
    R_t &= \sum_{i \in A_t} Y_{i,t}, 
    \end{split}
\end{equation}
where it is typically assumed that $\sigma_1$ and $\sigma_2$ are known. 
We choose the prior $\vgamma \sim \normal(\vmu_{\vgamma}, {\Cov}_{\vgamma})$ with parameters as known. 
For this instance, the posteriors can be derived explicitly (see Appendix \ref{sec:appendix_LMM}). 
Many other distributions (e.g., Bernoulli) and model assumptions (e.g., Gaussian process) can be formulated similarly, depending on the applications.




\subsection{MNL bandits for dynamic assortment optimization
}\label{sec:MNL_bandits}
Assortment optimization \citep{pentico2008assortment} is a long-standing problem that aims to solve the most profitable subset of items to offer, especially when there exist substitution effects. 
The Multinomial Logit (MNL) model \citep{luce2012individual} is arguably the most popular model, and the corresponding bandit problem has been studied, via either the feature-agnostic  \citep{agrawal2017thompson, agrawal2019mnl} or the feature-determined approaches \citep{ou2018multinomial, agrawal2020tractable}. 
 In assortment optimization, the agent needs to offer a subset (assortment) $A_t \in \mA = \{A\subseteq[N]:|A|\leq K\}$, and the customer will then choose either one of them or the no-purchase option (denoted as item $0$).
Let $\vY_t=(Y_{0,t},\cdots,Y_{N,t})^{T}$ be an indicator vector of length $N+1$, where $Y_{i,t}$ equals to $1$ if the item $i$ is chosen. 
Let $\veta = (\eta_1, \dots,\eta_{N})^{T}$, where $\eta_k$ is the revenue of the item $k$.
The collected revenue in round $t$ is then $R_t = \sum_{i\in A_t} Y_{i,t} \eta_{i}$. 
In an MNL bandit, each item $i$ has an utility factor $v_i$, and 
the choice behaviour is characterized by a multinomial distribution
$\boldsymbol{Y}_t \sim Multinomial(1, \frac{v_i \I(i \in \{0\}\cup A_t)}{1 + \sum_{j \in A_t} v_i})$, with the convention that $v_0 = 1$. 
When $v_i$'s are known, the optimal assortment can be solved via linear programming \citep{agrawal2017thompson}.

Since direct inference under this model is intractable due to the complex dependency of the reward distribution on $A_t$, 
an epoch-type offering \cite{agrawal2017thompson, agrawal2019mnl, dong2020multinomial} is more popular in the bandit literature, where we keep offering the same assortment $A^l$ in the $l$th epoch until the no-purchase appears. 
Under this setup, it is easier to work with the item-specific parameter $\theta_i = (1+v_i)^{-1}$ and consider the number of purchase for the item $i$ in each epoch, denoted as $Y_{i}^l$. 
Then, 
based on Lemma 2 in \cite{agrawal2017thompson}, 
it can be proven that 
$Y_{i}^l \sim Geometric(\theta_i), \forall i \in A^l$. 

The nice property of such a schedule is that the distributions do not depend on $A_t$ any longer. 
Besides, the geometric distribution has a nice conjugate relationship with the Beta distribution. 
As a concrete example of our framework, we can consider modeling the relationship between $\theta_i$ and $\vx_i$ with the following Beta-Geometric logistic model:
\begin{equation}\label{eqn1}
    \begin{split}
     \theta_i &\sim Beta \big(\frac{logistic(\vx_i^T \vgamma)+ 1}{2}, \psi \big) , \forall i \in [N],\\
    Y_{i}^l &\sim Geometric(\theta_i), \forall i \in A^l,\\
    R_l &= \sum_{i\in A^{l}}Y_{i}^l\eta_{i}
    \end{split}
\end{equation}
Other generalization models such as the one in \eqref{eqn:hierachical_model_cascading} are also possible. 
We choose this specific form as it is widely observed \cite{agrawal2017thompson, agrawal2019mnl, dong2020multinomial} that $v_i < 1$, i.e., no item is more popular than the no-purchase option. This is equal to $\theta_i \in (1/2, 1)$. 
Finally, we remark that Algorithm \ref{alg:main_TS} needs to be slightly modified to be consistent with the epoch-style offering, though the main idea remains exactly the same. 
We defer the full algorithm to Appendix \ref{sec:alg_MNL}. 
The choices of priors and the posterior updating rules are similar to Section \ref{sec:cascading}. 

\section{Theory}

In this section, we provide theoretical guarantees for \name{MTSS}. 
We start with a general result, i.e., we do not restrict to a specific problem. 
Our result is information-theoretic and the proof is inspired by \citet{lu2019information}. 

Let $I(X;Y)$ be the \textit{mutual information} \citep[MI,][]{kullback1997information} between two random variables $X$ and $Y$, $I(X;Y|Z)$ be the conditional MI conditioned on $Z$, and $I_t(X;Y) = I(X;Y|\{(A_t', \vY_t')\}_{t'=1}^{t-1})$. 
To save space, we defer the detailed definitions to the Appendix \ref{preliminary}. 
Intuitively, MI is a measure of the mutual dependence between two variables.

Let $\Delta_{t} = \max_a r(a, \vthe) - r(A_t, \vthe)$ be the per-round regret, and $\Mean_t(X) = \Mean(X|\{(A_t', \vY_t')\}_{t'=1}^{t-1})$. 
For a given problem, we assume we can first find some $\Gamma_t$ and $\epsilon_t$, such that 
\begin{equation}
\begin{split}
   \Mean_{t}[\Delta_{t}] \le \Gamma_{t}\sqrt{I_{t}(\vthe;A_{t},\vY_{t})}+\epsilon_{t}.
\label{eqn:Delta_decomp}
\end{split}
\end{equation}
Here, $\Gamma_t$ is related to the concentration property of the model, and $\epsilon_t$ is a small error term. They can typically be derived by following a few routines introduced in \citet{lu2019information}. 
We will give an example shortly. 

To gain more insights of our bound below, we introduce \textit{oracle-TS}, the TS algorithm that has access to the true generalization model \textit{a priori} and use $\{g(\theta_i|\vx_i, \vgamma)\}$ as priors in feature-agnostic TS. 
For a general structured bandit problem, the regret of \name{MTSS} can be bounded as follows. 

\begin{theorem}[General Regret Bound] \label{general regret bound}
\textit{Suppose that \eqref{eqn:Delta_decomp} holds for all $t$. 
Assume $\Gamma_t \le \Gamma$ almost surely for some $\Gamma$. Then for Algorithm \ref{alg:main_TS}, 
$BR(T)$ is bounded by} 
\scalebox{0.95}{\parbox{\linewidth}{%
\begin{align}\label{eqn:general_bound}
    \underbrace{\Gamma\sum_{t}\Mean[\sqrt{I_{t}(\vgamma;A_{t},\vY_{t})}]}_\text{Regret due to not knowing $\vgamma$}
    + \underbrace{\sum_{t}\Gamma\Mean[\sqrt{I_{t}(\vthe;A_{t},\vY_{t}|\vgamma)}] + \Mean[\epsilon_{t}]}_\text{Regret bound for Oracle-TS}. 
\end{align}
}}
\vspace{-.6cm}
\end{theorem}
This decomposition provides intuitive insights into the performance of \name{MTSS}. 
It is also consistent with our construction, as \name{MTSS} aims to learn the generalization model to perform closer to oracle-TS while minimizing the regret. 
The specific regret bound is highly problem-dependent. 
It depends on both the first part of \eqref{eqn:general_bound} which measures the cost of learning the meta parameter $\vgamma$ (or equivalently, learning the true prior for $\vthe$), and the second part which quantifies the regret of the agent with known $\vgamma$ (or equivalently, the performance of oracle-TS). 
For problems with existing results on Bayes regret bound for feature-agnostic TS, the latter can usually be derived with minimal modifications.

As a concrete example, we next analyze the combinatorial semi-bandits with the linear mixed model as the generalization function (see Section \ref{sec:semi-bandit}). 
This example demonstrates the benefits of meta-learning clearly. 
Without loss of generality, we first state several standard regularity conditions:

\textbf{Assumption 1.} $\|\vx_{i}\|_{2}\leq 1$, for all $i \in [N]$.

\textbf{Assumption 2.} $\Sigma_{\vgamma}, \sigma_{1}, \sigma_{2}$ are all bounded.


We have the following regret bound. 
\begin{theorem}[Semi-bandits]\label{Theorem1}
\textit{Under Assumptions 1-2, the Bayes regret of the \name{MTSS} under model \eqref{eqn:LMM} is bounded by }

\begin{align*}
    BR(T)
    &\leq 
    \underbrace{c_1 K\sqrt{Td}\sqrt{ log\Big(1+\frac{N\lambda_{1}(\Sigma_{\vgamma})}{\sigma_{1}^{2}+\sigma_{2}^{2}/T}\Big)}}_\text{Regret due to not knowing $\vgamma$}\\
    &{+} 
    \underbrace{c_1\sqrt{NTK}\sqrt{log(1 {+} \frac{\sigma_{1}^{2}}{\sigma_{2}}T)}
    {+} K\sqrt{\frac{2}{N}(\lambda_{1}(\Sigma_{\vgamma}) {+} \sigma_{1}^{2})}}_\text{Regret bound for Oracle-TS}\\
    &=\TO\Big(K\sqrt{Td} + \sqrt{NTK} \Big), 
\end{align*}

\textit{where $c_1 = 4\sqrt{\frac{\lambda_{1}(\Sigma_{\vgamma})+\sigma_{1}^{2}}{2log(1+(\lambda_{1}(\Sigma_{\vgamma})+\sigma_{1}^{2})/\sigma_{2}^{2})} log(4NT^{2})}$ and  $\TO$ is the big-O notation that hides logarithmic terms.} 
\end{theorem}
\vspace{-.2cm}
Therefore, when there is a large number of items ($Kd = o(N)$), the regret due to not knowing $\vgamma$ is asymptotically negligible (i.e., dominated by the second part), and the performance of \name{MTSS} is close to oracle-TS, as also observed in experiments. 
Moreover, note that the second part of the regret bound is dominated by 
$c_1\sqrt{NTK}\sqrt{log(1{+}\frac{\sigma_{1}^{2}}{\sigma_{2}}T)}$, 
which will decay to zero as $\sigma_1$ decreases, i.e., when the features become more useful. 
Therefore, we claim \name{MTSS} as scalable, since it allows utilizing feature information to learn shared structure so as to behave close to oracle-TS, which yields low regret when the features are informative and serves as the skyline. 


In contrast, 
as derived in \citet{basu2021no}, 
the additional Bayes regret of feature-agnostic TS than that of oracle-TS can only be bounded by $\sqrt{NTK}$. 
The dependency on $N$ is as expected, as feature-agnostic TS fails to share information across items and has to learn each from scratch. 
As such, \name{MTSS} would be more efficient when features are informative and the number of items is sufficient to learn a good generalization model ($Kd = o(N)$).
On the other hand, 
similar to the discussions in \citet{foster2020adapting} and  \citet{krishnamurthy2021tractable}, 
feature-determined TS assumes a restricted model and might suffer from the model misspecification. Therefore, to the best of our knowledge, one can expect a regret bound that is linear in $T$, which is consistent with our observations in experiments. 

\vspace{-.1cm}
\section{Experiments}\label{sec:expt}
\vspace{-.1cm}
\subsection{Synthetic datasets}\label{sec:simulation}
\textbf{Setting.} 
We first conduct simulation experiments to support our theoretical results and investigate the empirical performance of different approaches under various situations. 
We use the three models introduced in Section \ref{sec:example} to generate data, 
with $(N,K)$ set as $(1000, 3)$, $(3000, 10)$, $(1000, 5)$ for cascading bandits, semi-bandits, and MNL bandits, respectively. 
We set $d = 5$ for all tasks, $\veta = \vone$ for MNL bandits, and $\sigma_2 = 1$ for semi-bandits. 
We choose $Q(\vgamma) = \normal(\vzero, d^{-1}\vI)$ and sample $\vx_i$ from $\normal(\vzero, \vI)$ with an intercept. 
For each problem, we vary the value of either $\sigma_1$ or $\psi$, where a higher value of $\sigma_1$ or $\psi$ implies a larger heterogeneity between items, conditional on their features. 


\textbf{Baselines.}
For these three problems, we compare our framework with existing methods, which can be categorized as either feature-agnostic or feature-determined as we introduced. 
For the feature-agnostic approaches, we directly apply the TS algorithms proposed in the corresponding papers \citep{kveton2015cascading, wang2018thompson, agrawal2017thompson}, which are designed for each problem separately. 
For the feature-determined approaches, a naive application of some existing TS methods \citep{zong2016cascading, ou2018multinomial} will yield an unfair comparison, as they are based on functional assumptions on $g$ that are different from the data generation model used in this study (e.g., linear instead of logistic). 
Therefore, we closely follow their spirits and adapt in the following way: 
we use the models considered in Section \ref{sec:example}, but set $\sigma_1$ or $\psi$ as zero to derive the TS algorithms accordingly. 
We also present the performance of oracle-TS as our skyline. 
Finally, to study the performance of our algorithm with the offline-training-online-deployment schedule as in 
Section \ref{sec:offline_training}, we sample a new $\Tilde{\vgamma}$ every $500$ time points in MNL bandits and cascading bandits, and every $100$ time points in semi-bandits. 

\textbf{Results.}
The experiment results averaged over $50$ random seeds are presented in Figure \ref{fig:simu}. 
Overall, \name{MTSS} performs favorably and demonstrates its universality. 
Our findings can be summarized as follows. 
On one hand, \name{MTSS} enjoys a sublinear regret, while the feature-determined approach suffers from a linear regret due to the bias. 
This bias becomes more severe when $\sigma_1$ or $\phi$ increases, which implies that the amount of variation in $\theta_i$ that can not be explained by $g(\vx_i)$ grows. 
On the other hand, although the feature-agnostic methods in general have a sublinear regret, the learning speed is slow and hence the cumulative regret is much larger. 
This is due to the lack of generalization across items. 
With the offline-training-online-deployment schedule, our algorithm still performs well and is close to oracle-TS. 
Finally, \name{MTSS} is computationally efficient during online updating. 
For example, for MNL bandits, the total online time costs for feature-agnostic TS and \name{MTSS} are $2.3$ and $2.5$ seconds, respectively. 


\newcommand{\sharedwidthSubFig}{0.75}
\newcommand{\sharedwidth}{1}
\newcommand{\sharedheight}{2.8cm}
\newcommand{\vspacesmall}{-.1cm}

\begin{figure}[t]
     \centering
 \begin{subfigure}[b]{\sharedwidthSubFig\textwidth}
     \centering
     \includegraphics[width=\sharedwidth\textwidth, height = \sharedheight]{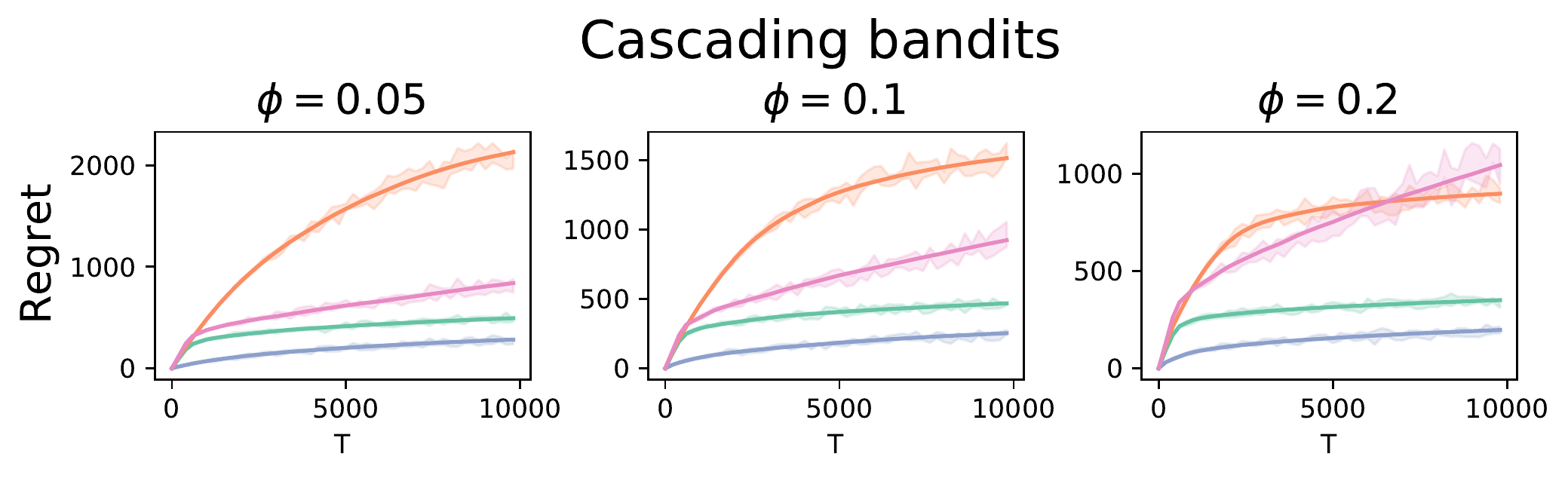}
    \vspace{\vspacesmall}
 \end{subfigure}
 \begin{subfigure}[b]{\sharedwidthSubFig\textwidth}
     \centering
     \includegraphics[width=\sharedwidth\textwidth, height = \sharedheight]{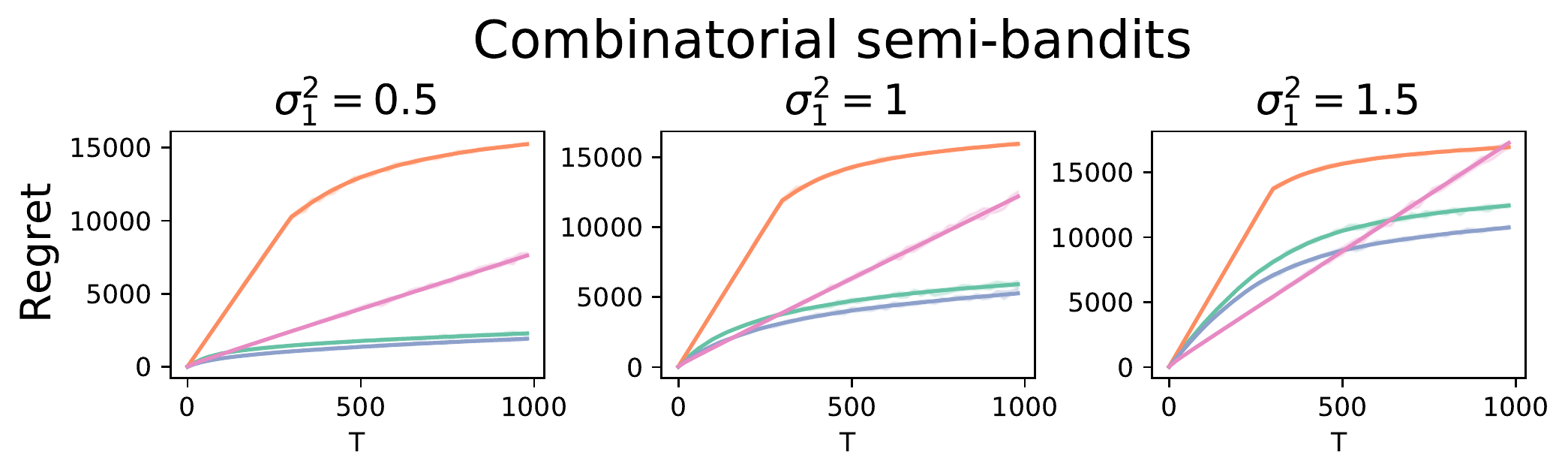}
    \vspace{\vspacesmall}
 \end{subfigure}
 \begin{subfigure}[b]{\sharedwidthSubFig\textwidth}
     \centering
     \includegraphics[width=\sharedwidth\textwidth, height = \sharedheight]{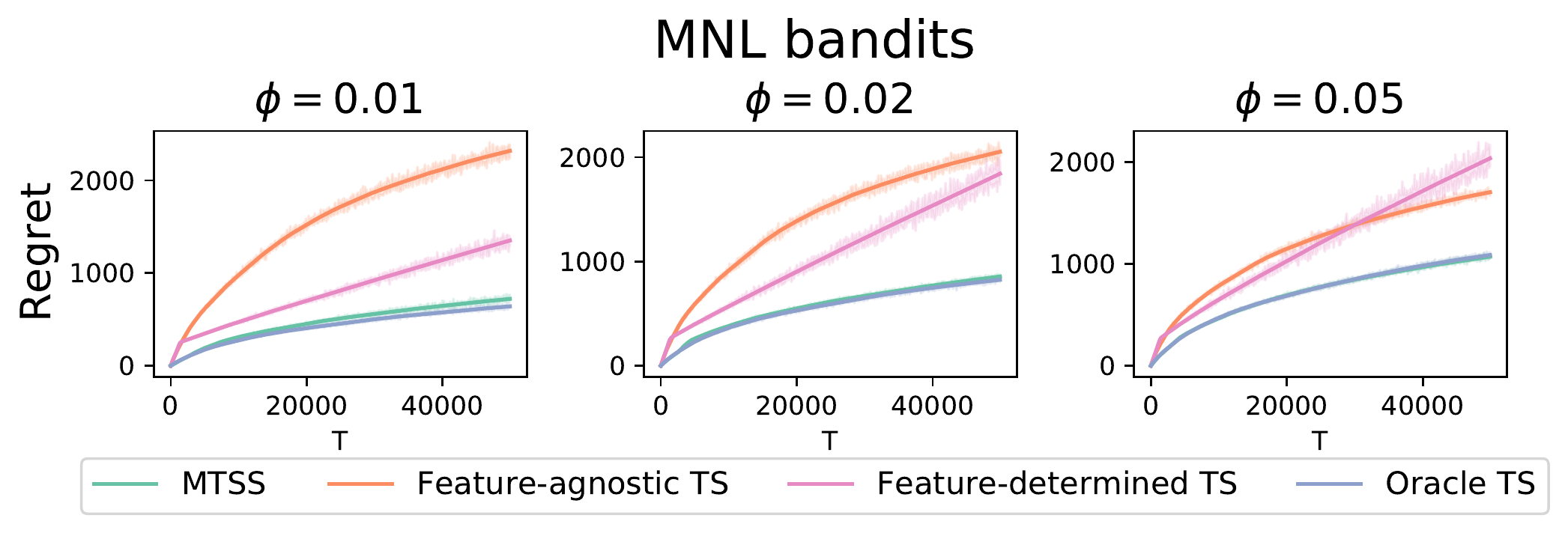}
 \end{subfigure}
\caption{
Simulation results. Shaded areas indicate the standard errors of the averages.}
\label{fig:simu}
\end{figure}


\textbf{Additional experiments.} 
First of all, we repeat the experiments with other values of $L, K, d$ in Appendix \ref{sec:appendix_additional_expt}, and the findings are similar. 
Second, we numerically study the impact of model misspecification in Appendix \ref{sec:appendix_additional_expt_robustness}, where \name{MTSS} shows great robustness. 
Recall that, to facilitate scalability, we assume that 
$\theta_i| \vx_i, \vgamma \sim g(\theta_i|\vx_i, \vgamma)$. 
Intuitively, since $g$ is used to construct a prior for the feature-agnostic model, 
as long as the learned priors provide reasonable information compared to the manually specified ones, this framework is still valuable. 
The robustness is also supported by the prior-independent or instance-independent sublinear regret bounds for feature-agnostic TS \citep{wang2018thompson, perrault2020statistical, zhong2021thompson}. 
Finally, 
in Appendix \ref{sec:appendix_additional_expt_cold_start}, 
we conduct experiments under the cold-start setting, where new items will be frequently introduced and old items will be removed. 
The benefit of \name{MTSS} is significant under this practical setting. 

\subsection{Real data}\label{sec:real_expt}
\begin{figure*}[t]
 \centering
 \includegraphics[width=0.8\textwidth]{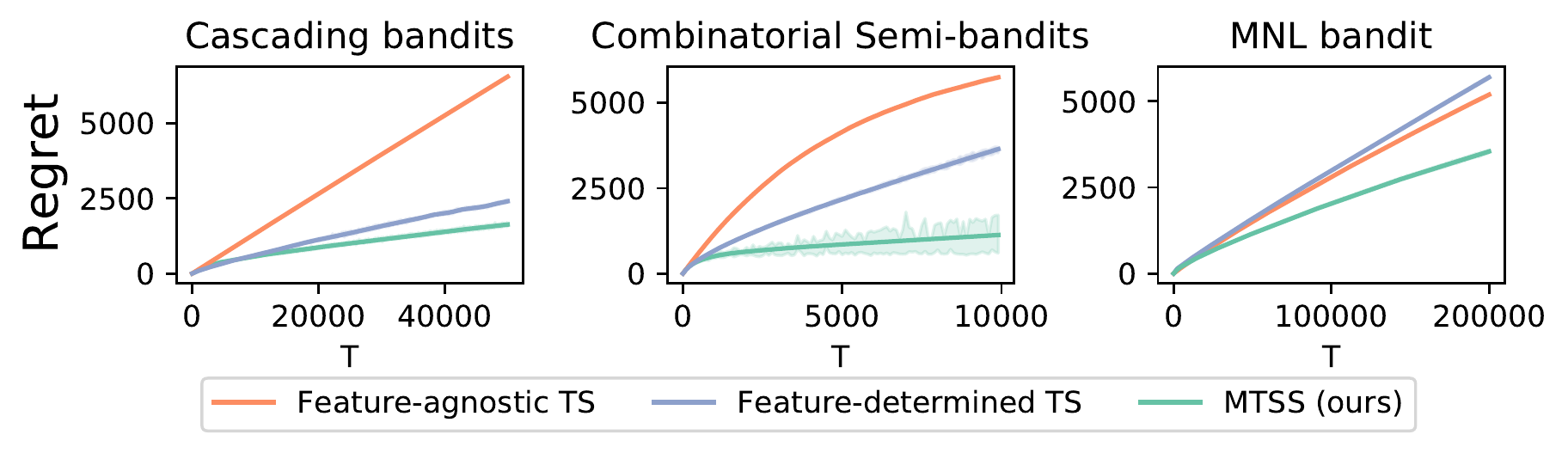} 
 \vspace{-0.4cm}
\caption{
Experiment results on real datasets, averaged over $50$ random seeds. 
Shared areas indicate the standard errors of the mean, which are small and hence hard to distinguish for some curves. }
\label{fig:Real Data}
\vspace{-.4cm}
\end{figure*}

In this section, we compare the proposed framework with existing methods (discussed in Section \ref{sec:simulation}) on three real datasets. 
For fair comparisons, we closely follow the related feature-determined papers \cite{zong2016cascading, wen2015efficient, harper2015movielens} to design our experiments.  
To save space, we describe the main ideas below and defer more details to Appendix \ref{sec:appendix_additional_expt_details}.

\textbf{Datasets.} 
For cascading bandits, we follow \citet{zong2016cascading} and aim to rank and display restaurants, using the dataset from Yelp dataset challenge \citep{asghar2016yelp}. 
In the final dataset, we display a set of $5$ restaurants from a universe of size $3000$, and utilize $10$ features. 
For combinatorial semi-bandits, we follow \citet{wen2015efficient} and aim to send online advertisements to the best subset of users while  keeping a balance between genders, using an income dataset from \citet{asuncion2007uci}. 
In the final dataset, we choose a set of $20$ users from a universe of size $3000$, and utilize $4$ user-specific features. 
For MNL bandits, we follow \citet{oh2019thompson} and aim to recommend the optimal set of movies, using the MovieLens dataset \citep{harper2015movielens}. 
In the final dataset, we recommend a set of $5$ movies from a universe of size $1000$, and utilize $5$ movie-specific features. 


\textbf{Design.} 
To simulate data and calculate regrets, we need to first determine $\{\theta_i\}$ and $\{\vx_i\}$, and then generate stochastic observations either by using the base model \eqref{eqn:main_raw_model} with $\{\theta_i\}$ as parameters or by directly sampling from the logged data. 
Again, we closely follow the existing papers. 
Specifically, for cascading bandits and MNL bandits, we follow a common schema, by first splitting the dataset into a training set and a testing set, 
and then estimating the features $\{\vx_i\}$ from the training set (via collaborative filtering), and finally estimating the item-specific parameters $\{\theta_i\}$ from the testing set. 
For semi-bandits, we directly utilize the features and responses in the raw dataset. 
We remark that, during these procedures, \textit{zero} assumption is imposed manually on the joint distribution of $(\vx_i, \theta_i)$ (and hence $\prob(\theta_i | \vx_i)$), and therefore these setups can be used for fair comparisons between \name{MTSS} and the existing approaches. 






\textbf{Results.} 
We present experiment results in Figure \ref{fig:Real Data}. 
\name{MTSS} accumulates lower regrets in all three problems. 
On one hand, we observe that feature-agnostic TS suffers the curse of dimensionality and learns slowly. 
In particular, for cascading or MNL bandits, since the click/purchase rates are low in the two datasets (i.e., useful feedback is sparse), feature-agnostic TS shows a (nearly) linear regret, as also observed in  \citet{zong2016cascading} and  \citet{harper2015movielens}. 
On the other hand, feature-determined TS, although may slightly outperform at the initial periods, eventually shows a (nearly) linear trend. 
This is mainly due to the bias from the restrictive model assumption. 




\vspace{-.2cm}
\section{Related Work}
\vspace{-.2cm}
\textbf{Structured bandits.}
Standard multi-armed bandits are not scalable to huge action space, and therefore researchers leverage structural information to generalize across actions, known as structured bandits \citep{van2020optimal}. 
Besides several stylized models such as the linear bandits \citep{chu2011contextual} and logistic bandits \citep{kveton2020randomized}, 
many practical problems depend on domain-specific models and can be summarized as model \eqref{eqn:main_raw_model}. 
Besides the two major approaches and related papers reviewed in the previous sections, 
\citet{yu2020graphical} also proposes a unified framework for a few structured bandit problems. 
However, this paper mainly focuses on unifying problems without introducing new models and related algorithms for each specific problem as we do.  
In addition, their approach is restricted to models with only binary variables. 


\textbf{Meta Bandits.}
Our work is related to the meta bandits literature \citep{kveton2021meta, hong2021hierarchical, wan2021metadata}, 
where the focus is on sharing knowledge across a large number of relatively simple bandit tasks, such as multi-armed bandits or linear bandits. None of them can be applied to our problem of one single large-scale structured bandit task. 
In addition, all existing papers, except for \citet{wan2021metadata}, only model the tasks as sampled from a simple feature-agnostic distribution and can not utilize valuable side information such as features in meta-learning, while \citet{wan2021metadata} is applicable only to $K$-armed bandits. 

\textbf{Hierarchical modeling.} 
Our model belongs to hierarchical models \citep{gelman2006data} and the idea of constructing an informative prior is connected with the empirical Bayes approach \citep{maritz2018empirical}. 
Besides the several works on meta bandits cited above, most literature focuses on other areas such as supervised learning  \citep{luo2018mixed, shi2012mixed, wang2012nonparametric}. 
Our work opens a door to connecting the bandit problems with the rich literature on Bayesian hierarchical models. 



\vspace{-.2cm}
\section{Discussion}
\vspace{-.1cm}
In this paper, motivated by the challenges of online learning in large-scale structured bandits, we propose a unified meta-learning framework with a TS-type algorithm named \name{MTSS}. 
We demonstrate that the framework is general, scalable, and robust. 
The proposed framework can be extended in several aspects. 
First, it is straightforward to allow multiple parameters related to each item, by fitting one generalization model for each. 
Second, in our examples, we consider the variance components ($\sigma_1$ or $\psi$) as known. 
In practice, we can apply empirical Bayes  to update these hyperparameters adaptively (see Appendix \ref{sec:EB_variable}). 
Third, we mainly focus on TS algorithms as our baselines, because they typically outperform the UCB counterpart and yield fair comparisons with \name{MTSS}. 
Although it is not straightforward to adapt UCB to our problem, Bayesian UCB \citep{kaufmann2012bayesian} can be similarly developed. 




\clearpage


\bibliographystyle{icml2022.bst}
\bibliography{0_main}

\onecolumn

\clearpage
\appendix
\newpage


\section{Algorithm details and extensions}
\subsection{Empirical Bayes for updating variance components adaptively}\label{sec:EB_variable}
In our examples, we consider the variance components ($\sigma_1$ or $\psi$) as known. 
In practice, we can apply empirical Bayes \cite{maritz2018empirical} to update these hyperparameters adaptively, as in \citet{tomkins2019intelligent} and  \citet{wan2021metadata}. 
Specifically, suppose the generalization model is $\theta_i| \vx_i, \vgamma  \sim g(\theta_i|\vx_i, \vgamma; \beta)$, where $\beta$ is a parameter that we assume as known in MTSS. 
At time point $t$, given the history $\mH_t$, 
one can focus on the following frequentist model:  
\begin{equation}
  \begin{alignedat}{2}
&\text{(Generalization function)} \;
\;    \theta_i| \vx_i, \vgamma  &&\sim g(\theta_i|\vx_i, \vgamma; \vbeta), \forall i \in [N],\\ 
&\text{(Observations)} \quad\quad\quad\quad\quad\quad\;
\;    \boldsymbol{Y}_t &&\sim f(\boldsymbol{Y}_t|A_t, \vthe),\\
&\text{(Reward)} \quad\quad\quad\quad\quad\quad\quad\quad\;
\;   R_t &&= f_r(\boldsymbol{Y}_t ; \vr). 
      \end{alignedat}
\end{equation}
We write the corresponding likelihood function as 
$L(\vthe, \vgamma, \beta| \mH_t)$, and let $(\hat{\vthe}, \hat{\vgamma}, \hat{\beta})$ be the maximum likelihood estimation. 
Following the empirical Bayes approach, we use $\theta_i| \vx_i, \vgamma  \sim g(\theta_i|\vx_i, \vgamma; \hat{\beta})$ in MTSS. 
The updating of $\hat{\beta}$ can also be periodical. 

Intuitively, when the conditional variance decays to $0$, our method reduces to feature-determined TS; 
while when it grows, it indicates the features are less useful, and we are essentially assigning a  non-informative prior as commonly adopted in feature-agnostic TS. 
As such, our framework yields the desired flexibility and is adaptive via empirical Bayes. 

\subsection{Algorithm for MNL bandits}\label{sec:alg_MNL}
As an exmaple, we adopt an epoch-style offering for MNL banidts in the main text. 
With this schedule, Algorithm \ref{alg:main_TS} needs be slightly modified to be consistent, 
though the main idea remains exactly the same. 
We present the modified MTSS in Algorithm \ref{alg:MNL_TS}. 
The only difference is that our schedule of sampling new parameters is adjusted to be consistent with the epoch-style.

\begin{algorithm}[!h]
\SetKwData{Left}{left}\SetKwData{This}{this}\SetKwData{Up}{up}
\SetKwFunction{Union}{Union}\SetKwFunction{FindCompress}{FindCompress}
\SetKwInOut{Input}{Input}\SetKwInOut{Output}{output}
\SetAlgoLined
\Input{
Prior $\prob(\vgamma)$ and known parameters of the hierarchical model
}
Set $\mH_{1} = \{\}$, $t$=1, and $l$=1 keeps track of the time steps and total number of epochs, respectively.

\While{$t<T$}{ 
Compute the posterior distribution $\prob(\vthe | \mH_{l})$\\
For each item $i=1,\cdots,N$, sample $\Tilde{\vthe}$ from  $\prob(\vthe | \mH_{l})$, and compute the utility $\Tilde{v}_{i}=\frac{1}{\Tilde{\theta}_{i}}-1$\\
Compute $A^{l} = \argmax_{a \in \mA} \Mean(R_t \mid a, \tilde{\vthe})$;\\

\While{$c_{t}\neq 0$}{
Offer $A^{l}$, observe the purchasing decision $c_{t}$ of the consumer\\
Update $\xi_{l}=\xi_{l}\cup t$, time indices corresponding to epoch $l$\\
$t = t + 1$\\
}
For each item $i\in A^{l}$, compute $Y_{i}^{l}=\sum_{t\in \xi_{l}}I(c_{t}=i)$, which is the number of picks of item i in epoch $l$\\
Update the dataset as $\mH_{l+1} \leftarrow \mH_{l} \cup \{(A^{l}, \{Y_{i}^{l}\})\}$\\
$l = l+1$

}
 \caption{\name{MTSS} with epoch-type schedule for MNL bandits}\label{alg:MNL_TS}
\end{algorithm}

\subsection{Explicit form of the posterior in semi-bandits with LMM}\label{sec:appendix_LMM}
In this section, we derive the posterior distributions involved in the algorithm for semi-bandits and the proof of Theorem \ref{Theorem1}. The derivations are standard, and we only include them for completeness. 

Recall that $\vx_{i}$ is the features of item $i$. Let $\vPhi = (\vx_{1},\cdots,\vx_{N})^{T}$ contains all $N$ items' features. Let $\vphi_{t}=(\vx_{k})_{k\in A_{t}}^{T}$ be a $|A_{t}|\times d$ matrix contains features of all items offered at round $t$, and $\vPhi_{1:t}=(\vphi_{1}^{T},\cdots,\vphi_{t}^{T})^{T}$ is a $C_{t} \times d$ matrix including features of all the item offered from round $1$ to round $t$, where $C_{t}=\sum_{l=1}^{t}|A_{l}|$. Likewise, $\vY_{1:t} = (\vY_{1}^{T},\cdots,\vY_{t}^{T})^{T}$ includes observed rewards of all items offered till round $t$. Then, we define a $N\times C_{t}$ matrix $\vZ_{1:t}$, such that the $(j,a)$-th entry of $\vZ_{1:t}$ is $\mathbbm{I}(i(a) = j)$, $j \in [N]$. Here, $i(a)$ is the item index of the $a$th observed reward in $\vY_{1:t}$. The row $i$ of $\vZ_{1:t}$ is defined as $\vZ_{1:t,i}$. Finally, we define that $n_{t}(i)$ is the total number of pulls of arm $i$ from round $1$ till round $t$, include round $t$.

Recall that our model for semi-bandits is defined as following:
\begin{equation*}
    \begin{split}
    \vgamma &\sim \normal(\vmu_{\vgamma}, {\Cov}_{\vgamma}),\\
     \theta_i &\sim \normal(\vx_i^T \vgamma, \sigma_1^2), \forall i \in [N],\\
    Y_{i, t} &\sim \normal(\theta_i, \sigma_2^2), \forall i \in A_{t}.\\
    \end{split}
\end{equation*}

\textbf{Posterior Distribution of $\vthe$ Given $\mH_{t+1}$:}

First, we compute the distribution of $\vY_{1:t}$ given $\vPhi_{1:t}(\vZ_{1:t})$ and $\vthe$, the distribution of $\vY_{1:t}$ given only $\vPhi_{1:t}$, and the distribution of $\vthe$. Note that, $\vPhi_{1:t} = \vZ_{1:t}^{T}\vPhi$. Given $\vthe$, we can write 
\begin{align*}
    \vY_{1:t} = \vZ_{1:t}^{T}\vthe + \epsilon \text{, where } \epsilon \sim \normal(0,\sigma_{2}^{2}\vI_{C_{t}}).
\end{align*}
Similarly, given $\vPhi$ and $\vgamma$, we have 
\begin{align*}
    \vthe = \vPhi\vgamma + v \text{, where } v \sim \normal(0,\sigma_{1}^{2}\vI_{N}).
\end{align*}
Further, given $\vmu_{\vgamma}$, we have
\begin{align*}
    \vgamma = \vmu_{\vgamma} + b   \text{, where } b \sim \normal(0,{\Cov}_{\vgamma}).
\end{align*}
Combining above three equations, we have
\begin{align*}
    \vY_{1:t}|\vPhi_{1:t},\vthe &= \vZ_{1:t}^{T}\vthe + \epsilon,\\
    \vthe &= \vPhi\vmu_{\vgamma} + \vPhi b + v,\\
    \vY_{1:t}|\vPhi_{1:t} &= \vPhi_{1:t}\vmu_{\vgamma} + \vPhi_{1:t} b + \vZ_{1:t}^{T}v + \epsilon.
\end{align*}
Therefore, we have
\begin{align*}
    \vY_{1:t}|\vPhi_{1:t},\vthe &\sim  \normal(\vZ_{1:t}^{T}\vthe, \sigma_{2}^{2}\vI_{C_{t}}),\\
    \vthe &\sim \normal(\vPhi\vmu_{\vgamma},\vPhi{\Cov}_{\vgamma}\vPhi^{T}+\sigma_{1}^{2}\vI_{N}), \\
    \vY_{1:t}|\vPhi_{1:t} &\sim \normal(\vPhi_{1:t}\vmu_{\vgamma}, \vPhi_{1:t}{\Cov}_{\vgamma}\vPhi^{T}_{1:t}+\sigma_{1}^{2}\vZ_{1:t}^{T}\vZ_{1:t}+\sigma_{2}^{2}\vI_{C_{t}}).
\end{align*}
Let $\vB \sim \normal(0,\vPhi{\Cov}_{\vgamma}\vPhi^{T}+\sigma_{1}^{2}\vI_{N})$, we have
\begin{align*}
    \vthe &= \vPhi\vmu_{\vgamma} + \vB,\\
    \vY_{1:t}|\vPhi_{1:t},\vB &\sim \normal(\vPhi_{1:t}\vmu_{\vgamma}+\vZ_{1:t}^{T}\vB,\sigma_{2}^{2}\vI_{C_{t}}).
\end{align*}

Then, we compute the posterior distribution of $\vB$ instead of $\vthe$.

\begin{align*}
    \mathbbm{P}(\vB|\vY_{1:t}) &\propto \mathbbm{P}(\vY_{1:t}|\vB)\mathbbm{P}(\vB)\\
    &\propto exp\Big(-\frac{1}{2}(\vY_{1:t}-\vPhi_{1:t}\vmu_{\vgamma}-\vZ_{1:t}^{T}\vB)^{T}\frac{1}{\sigma_{2}^{2}}(\vY_{1:t}-\vPhi_{1:t}\vmu_{\vgamma}-\vZ_{1:t}^{T}\vB)\Big)exp\Big(-\frac{1}{2}\vB^{T}(\vPhi{\Cov}_{\vgamma}\vPhi^{T}+\sigma_{1}^{2}\vI_{N})^{-1}\vB\Big)\\
    &\propto exp\Big(\vB^{T}\vZ_{1:t}\frac{1}{\sigma_{2}^{2}}(\vY_{1:t}-\vPhi_{1:t}\vmu_{\vgamma})-\frac{1}{2}\vB^{T}\underbrace{\{(\vPhi{\Cov}_{\vgamma}\vPhi^{T}+\sigma_{1}^{2}\vI_{N})^{-1}+\frac{1}{\sigma_{2}^{2}}\vZ_{1:t}\vZ_{1:t}^{T}\}}_{\tilde{{\Cov}}^{-1}}\vB\Big)\\
    &\sim \normal(\underbrace{\tilde{{\Cov}}\vZ_{1:t}\frac{1}{\sigma_{2}^{2}}(\vY_{1:t}-\vPhi_{1:t}\vmu_{\vgamma})}_{\vmu(\vB)},\tilde{{\Cov}}).
\end{align*}
Using the Woodbury matrix identity \cite{rasmussen2003gaussian}, we have
\begin{align*}
    \tilde{{\Cov}} &= \vPhi{\Cov}_{\vgamma}\vPhi^{T}+\sigma_{1}^{2}\vI_{N}-(\vPhi{\Cov}_{\vgamma}\vPhi_{1:t}^{T}+\sigma_{1}^{2}\vZ_{1:t})\Big(\sigma_{2}^{2}\vI_{C_{t}}+\vPhi_{1:t}{\Cov}_{\vgamma}\vPhi_{1:t}^{T}+\sigma_{1}^{2}\vZ_{1:t}^{T}\vZ_{1:t}\Big)^{-1}(\vPhi{\Cov}_{\vgamma}\vPhi_{1:t}^{T}+\sigma_{1}^{2}\vZ_{1:t})^{T},
\end{align*}
and 
\begin{align*}
    \vmu(\vB) &= \tilde{{\Cov}}\vZ_{1:t}\frac{1}{\sigma_{2}^{2}}(\vY_{1:t}-\vPhi_{1:t}\vmu_{\vgamma})\\
    & = (\vPhi{\Cov}_{\vgamma}\vPhi_{1:t}^{T}+\sigma_{1}^{2}\vZ_{1:t})\Big(\vI_{C_{t}}-\Big(\sigma_{2}^{2}\vI_{C_{t}}+\vPhi_{1:t}{\Cov}_{\vgamma}\vPhi_{1:t}^{T}+\sigma_{1}^{2}\vZ_{1:t}^{T}\vZ_{1:t}\Big)^{-1}(\vPhi{\Cov}_{\vgamma}\vPhi_{1:t}^{T}+\sigma_{1}^{2}\vZ_{1:t})^{T}\vZ_{1:t}\Big)\frac{1}{\sigma_{2}^{2}}(\vY_{1:t}-\vPhi_{1:t}\vmu_{\vgamma})\\
    & = (\vPhi{\Cov}_{\vgamma}\vPhi_{1:t}^{T}+\sigma_{1}^{2}\vZ_{1:t})\Big(\sigma_{2}^{2}\vI_{C_{t}}+\vPhi_{1:t}{\Cov}_{\vgamma}\vPhi_{1:t}^{T}+\sigma_{1}^{2}\vZ_{1:t}^{T}\vZ_{1:t}\Big)^{-1}\sigma_{2}^{2}\vI_{C_{t}}\frac{1}{\sigma_{2}^{2}}(\vY_{1:t}-\vPhi_{1:t}\vmu_{\vgamma})\\
    & = (\vPhi{\Cov}_{\vgamma}\vPhi_{1:t}^{T}+\sigma_{1}^{2}\vZ_{1:t})\Big(\sigma_{2}^{2}\vI_{C_{t}}+\vPhi_{1:t}{\Cov}_{\vgamma}\vPhi_{1:t}^{T}+\sigma_{1}^{2}\vZ_{1:t}^{T}\vZ_{1:t}\Big)^{-1}(\vY_{1:t}-\vPhi_{1:t}\vmu_{\vgamma}).
\end{align*}
Since $\vthe = \vPhi\vmu_{\vgamma} + \vB$, we get the posterior distribution of $\vthe$.
\begin{align*}
    \vthe|\mH_{t+1} \sim \normal(\vPhi\vmu_{\vgamma}+\vmu(\vB),\tilde{{\Cov}}).
\end{align*}
In particular, for each item $i \in [N]$, the posterior distribution of the item-specific parameter $\theta_{i}$ is as follows.
\begin{equation}
    \begin{split}
\theta_{i}|\mH_{t+1} &\sim \mathcal{N}(\hat{\mu}_{t+1}(i),\hat{\sigma}_{t+1}^{2}(i)),\\
\hat{\mu}_{t+1}(i)&=\vx_i^{T}\vmu_{\vgamma}+(\vx_i{\Cov}_{\vgamma}\vPhi_{1:t}^{T}+\sigma_{1}^{2}\vZ_{1:t,i})(\sigma_{2}^{2}\vI_{C_{t}}+\vPhi_{1:t}{\Cov}_{\vgamma}\vPhi^{T}_{1:t}+\sigma_{1}^{2}\vZ_{1:t}^{T}\vZ_{1:t})^{-1}(\vY_{1:t}-\vPhi_{1:t} \vmu_{\vgamma}),\\
    \hat{\sigma}_{t+1}^{2}(i)&=\vx_i^T{\Cov}_{\vgamma}\vx_i+\sigma_{1}^{2}-(\vx_i^T{\Cov}_{\vgamma}\vPhi^{T}_{1:t}+\sigma_{1}^{2}\vZ_{1:t,i})(\sigma_{2}^{2}\vI_{C_{t}}+\vPhi_{1:t}{\Cov}_{\vgamma}\vPhi^{T}_{1:t}+\sigma_{1}^{2}\vZ_{1:t}^{T}\vZ_{1:t})^{-1}(\vx_i^T{\Cov}_{\vgamma}\vPhi^{T}_{1:t}+\sigma_{1}^{2}\vZ_{1:t,i})^{T}.
    \end{split}
\end{equation}
Alternatively, since
\begin{align*}
    \tilde{{\Cov}}^{-1} & = (\vPhi{\Cov}_{\vgamma}\vPhi^{T}+\sigma_{1}^{2}\vI_{N})^{-1}+\frac{1}{\sigma_{2}^{2}}\vZ_{1:t}\vZ_{1:t}^{T}\\
    &=(\vPhi{\Cov}_{\vgamma}\vPhi^{T}+\sigma_{1}^{2}\vI_{N})^{-1}+\frac{1}{\sigma_{2}^{2}}diag(n_{t}(i))_{i=1}^{N},
\end{align*}
then, 
\begin{align*}
    \hat{\sigma}_{t+1}^{-2}(i) = \hat{\sigma}_{t}^{-2}(i)+\frac{1}{\sigma_{2}^{2}}(n_{t}(i)-n_{t-1}(i)).
\end{align*}

\textbf{Posterior Distribution of $\vgamma$ Given $\mH_{t+1}$:} 

Similarly, we can write
\begin{align*}
    \vY_{1:t}|\vPhi_{1:t},b \sim \normal (\vPhi_{1:t}\vmu_{\vgamma}+\vPhi_{1:t}b, \sigma_{1}^{2}\vZ_{1:t}^{T}\vZ_{1:t}+\sigma_{2}^{2}\vI_{C_{t}}).
\end{align*}
Then we compute the posterior distribution of b instead of $\vgamma$.
\begin{align*}
    \mathbbm{P}(b|\vY_{1:t}) &\propto \mathbbm{P}(\vY_{1:t}|b)\mathbbm{P}(b)\\
    &\propto exp\Big(-\frac{1}{2}(\vY_{1:t}-\vPhi_{1:t}\vmu_{\vgamma}-\vPhi_{1:t}b)^{T}(\sigma_{1}^{2}\vZ_{1:t}^{T}\vZ_{1:t}+\sigma_{2}^{2}\vI_{C_{t}})^{-1}(\vY_{1:t}-\vPhi_{1:t}\vmu_{\vgamma}-\vPhi_{1:t}b)\Big)exp\Big(\frac{1}{2}b^{T}{\Cov}_{\vgamma}^{-1}b\Big)\\
    &\propto exp\Big(-\frac{1}{2}b^{T}\underbrace{(\vPhi_{1:t}^{T}(\sigma_{1}^{2}\vZ_{1:t}^{T}\vZ_{1:t}+\sigma_{2}^{2}\vI_{C_{t}})^{-1}\vPhi_{1:t}+{\Cov}_{\vgamma}^{-1})}_{{\Cov}_{*}^{-1}}b+b^{T}\vPhi_{1:t}^{T}(\sigma_{1}^{2}\vZ_{1:t}^{T}\vZ_{1:t}+\sigma_{2}^{2}\vI_{C_{t}})^{-1}(\vY_{1:t}-\vPhi_{1:t}\vmu_{\vgamma})\Big)\\
    &\sim \normal(\underbrace{{\Cov}_{*}\vPhi_{1:t}^{T}(\sigma_{1}^{2}\vZ_{1:t}^{T}\vZ_{1:t}+\sigma_{2}^{2}\vI_{C_{t}})^{-1}(\vY_{1:t}-\vPhi_{1:t}\vmu_{\vgamma})}_{\vmu_{*}},{\Cov}_{*}).
\end{align*}
Using the Woodbury matrix identity \cite{rasmussen2003gaussian}, we have
\begin{align*}
    {\Cov}_{*} = {\Cov}_{\vgamma}-{\Cov}_{\vgamma}\vPhi_{1:t}^{T}(\sigma_{1}^{2}\vZ_{1:t}^{T}\vZ_{1:t}+\sigma_{2}^{2}\vI_{C_{t}}+\vPhi_{1:t}{\Cov}_{\vgamma}\vPhi_{1:t}^{T})^{-1}\vPhi_{1:t}{\Cov}_{\vgamma},
\end{align*}
and
\begin{align*}
    {\vmu}_{*} &= {\Cov}_{*}\vPhi_{1:t}^{T}(\sigma_{1}^{2}\vZ_{1:t}^{T}\vZ_{1:t}+\sigma_{2}^{2}\vI_{C_{t}})^{-1}(\vY_{1:t}-\vPhi_{1:t}\vmu_{\vgamma})\\
    &= \Big({\Cov}_{\vgamma}-{\Cov}_{\vgamma}\vPhi_{1:t}^{T}(\sigma_{1}^{2}\vZ_{1:t}^{T}\vZ_{1:t}+\sigma_{2}^{2}\vI_{C_{t}}+\vPhi_{1:t}{\Cov}_{\vgamma}\vPhi_{1:t}^{T})^{-1}\vPhi_{1:t}{\Cov}_{\vgamma}\Big)\vPhi_{1:t}^{T}(\sigma_{1}^{2}\vZ_{1:t}^{T}\vZ_{1:t}+\sigma_{2}^{2}\vI_{C_{t}})^{-1}(\vY_{1:t}-\vPhi_{1:t}\vmu_{\vgamma})\\
    &= {\Cov}_{\vgamma}\vPhi_{1:t}^{T}(\vI-(\sigma_{1}^{2}\vZ_{1:t}^{T}\vZ_{1:t}+\sigma_{2}^{2}\vI_{C_{t}}+\vPhi_{1:t}{\Cov}_{\vgamma}\vPhi_{1:t}^{T})^{-1}\vPhi_{1:t}{\Cov}_{\vgamma}\vPhi_{1:t}^{T})(\sigma_{1}^{2}\vZ_{1:t}^{T}\vZ_{1:t}+\sigma_{2}^{2}\vI_{C_{t}})^{-1}(\vY_{1:t}-\vPhi_{1:t}\vmu_{\vgamma})\\
    &= {\Cov}_{\vgamma}\vPhi_{1:t}^{T}(\sigma_{1}^{2}\vZ_{1:t}^{T}\vZ_{1:t}+\sigma_{2}^{2}\vI_{C_{t}}+\vPhi_{1:t}{\Cov}_{\vgamma}\vPhi_{1:t}^{T})^{-1}(\sigma_{1}^{2}\vZ_{1:t}^{T}\vZ_{1:t}+\sigma_{2}^{2}\vI_{C_{t}})(\sigma_{1}^{2}\vZ_{1:t}^{T}\vZ_{1:t}+\sigma_{2}^{2}\vI_{C_{t}})^{-1}(\vY_{1:t}-\vPhi_{1:t}\vmu_{\vgamma})\\
    &= {\Cov}_{\vgamma}\vPhi_{1:t}^{T}(\sigma_{1}^{2}\vZ_{1:t}^{T}\vZ_{1:t}+\sigma_{2}^{2}\vI_{C_{t}}+\vPhi_{1:t}{\Cov}_{\vgamma}\vPhi_{1:t}^{T})^{-1}(\vY_{1:t}-\vPhi_{1:t}\vmu_{\vgamma}).\\
\end{align*}
To derive an explicit form of ${\Cov}_{*}$, we focus on $\vPhi_{1:t}^{T}(\sigma_{1}^{2}\vZ_{1:t}^{T}\vZ_{1:t}+\sigma_{2}^{2}\vI_{C_{t}})^{-1}\vPhi_{1:t}+{\Cov}_{\vgamma}^{-1}$. Again, using the Woodbury matrix identity, we have
\begin{align*}
    (\sigma_{1}^{2}\vZ_{1:t}^{T}\vZ_{1:t}+\sigma_{2}^{2}\vI_{C_{t}})^{-1}&=\sigma_{2}^{-2}\vI_{C_{t}}-\sigma_{2}^{-4}\vZ_{1:t}^{T}(\sigma_{1}^{-2}\vI_{N}+\sigma_{2}^{-2}\vZ_{1:t}\vZ_{1:t}^{T})^{-1}\vZ_{1:t}\\
    &=\sigma_{2}^{-2}\vI_{C_{t}}-\sigma_{2}^{-4}\vZ_{1:t}^{T}diag\Big(\frac{1}{\sigma_{1}^{-2}+\sigma_{2}^{-2}n_{t}(i)}\Big)_{i=1}^{N}\vZ_{1:t}.
\end{align*}
and
\begin{align*}
    {\Cov}_{*}^{-1} &= {\Cov}_{\vgamma}^{-1}+\vPhi_{1:t}^{T}(\sigma_{1}^{2}\vZ_{1:t}^{T}\vZ_{1:t}+\sigma_{2}^{2}\vI_{C_{t}})^{-1}\vPhi_{1:t}\\
    &={\Cov}_{\vgamma}^{-1}+\vPhi_{1:t}^{T}\Big(\sigma_{2}^{-2}\vI_{C_{t}}-\sigma_{2}^{-4}\vZ_{1:t}^{T}diag\Big(\frac{1}{\sigma_{1}^{-2}+\sigma_{2}^{-2}n_{t}(1)},\cdots,\frac{1}{\sigma_{1}^{-2}+\sigma_{2}^{-2}n_{t}(N)}\Big)\vZ_{1:t}\Big)\vPhi_{1:t}\\
    &={\Cov}_{\vgamma}^{-1}+\sigma_{2}^{-2}\vPhi_{1:t}^{T}\vPhi_{1:t}-\sigma_{2}^{-4}\vPhi_{1:t}^{T}\vZ_{1:t}^{T}diag\Big(\frac{1}{\sigma_{1}^{-2}+\sigma_{2}^{-2}n_{t}(1)},\cdots,\frac{1}{\sigma_{1}^{-2}+\sigma_{2}^{-2}n_{t}(N)}\Big)\vZ_{1:t}\vPhi_{1:t}\\
    &={\Cov}_{\vgamma}^{-1}+\sum_{i=1}^{N}\frac{n_{t}(i)}{\sigma_{2}^{2}+\sigma_{1}^{2}n_{t}(i)}\vx_i\vx_i^{T}.
\end{align*}
Therefore, 
\begin{equation}
    \begin{split}
\vgamma|\mH_{t+1} &\sim \mathcal{N}(\tilde{\vmu}_{t+1},\tilde{{\Cov}}_{t+1}),\\
\tilde{\mu}_{t+1} &= \vmu_{\vgamma}+ {\Cov}_{\vgamma}\vPhi_{1:t}^{T}(\sigma_{1}^{2}\vZ_{1:t}^{T}\vZ_{1:t}+\sigma_{2}^{2}\vI_{C_{t}}+\vPhi_{1:t}{\Cov}_{\vgamma}\vPhi_{1:t}^{T})^{-1}(\vY_{1:t}-\vPhi_{1:t}\vmu_{\vgamma}),\\
\tilde{{\Cov}}_{t+1}^{-1} &= {\Cov}_{\vgamma}^{-1}+\sum_{i=1}^{N}\frac{n_{t}(i)}{\sigma_{2}^{2}+\sigma_{1}^{2}n_{t}(i)}\vx_i\vx_i^{T}.
    \end{split}
\end{equation}

\textbf{Posterior Distribution of $\vthe$ Given $\mH_{t+1}$ and $\vgamma$:} 

Similarly, to derive the posterior distribution $\vthe$ given $\mH_{t+1}$ and $\vgamma$, we first derive the posterior distribution of $v$ given $\mH_{t+1}$ and $\vgamma$. Here, we can write 
\begin{align*}
    \vY_{1:t}|\vPhi_{1:t},\vgamma, v \sim \normal(\vPhi_{1:t}\vgamma+\vZ_{1:t}^{T}v,\sigma_{2}^{2}\vI_{C_{t}}).
\end{align*}
Then, the posterior distribution of $v$ given $\mH_{t+1}$ and $\vgamma$ is
\begin{align*}
    \mathbbm{P}(v|\vY_{1:t},\vgamma) &\propto \mathbbm{P}(\vY_{1:t}|v,\vgamma)\mathbbm{P}(v|\vgamma)\\
    &\propto exp\Big(-\frac{1}{2}\sigma_{2}^{-2}(\vY_{1:t}-\vPhi_{1:t}\vgamma-\vZ_{1:t}^{T}v)^{T}(\vY_{1:t}-\vPhi_{1:t}\vgamma-\vZ_{1:t}^{T}v)\Big) exp\Big(-\frac{1}{2}\sigma_{1}^{-2}v^{T}v\Big)\\
    &\propto exp\Big(-\frac{1}{2}v^{T}(\sigma_{2}^{-2}\vZ_{1:t}\vZ_{1:t}^{T}+\sigma_{1}^{-2}\vI_{N})v+v^{T}\sigma_{2}^{-2}\vZ_{1:t}(\vY_{1:t}-\vPhi_{1:t}\vgamma)\Big)\\
    &\sim \normal(\underbrace{(\sigma_{2}^{-2}\vZ_{1:t}\vZ_{1:t}^{T}+\sigma_{1}^{-2}\vI_{N})^{-1}\sigma_{2}^{-2}\vZ_{1:t}(\vY_{1:t}-\vPhi_{1:t}\vgamma)}_{\vmu_{**}},(\sigma_{2}^{-2}\vZ_{1:t}\vZ_{1:t}^{T}+\sigma_{1}^{-2}\vI_{N})^{-1}).
\end{align*}
Using the Woodbury matrix identity, we have
\begin{align*}
    \vmu_{**} = \sigma_{1}^{2}\vZ_{1:t}(\sigma_{2}^{2}\vI_{C_{t}}+\sigma_{1}^{2}\vZ_{1:t}^{T}\vZ_{1:t})^{-1}(\vY_{1:t}-\vPhi_{1:t}\vgamma).
\end{align*}
Furthermore, 
\begin{align*}
    \sigma_{2}^{-2}\vZ_{1:t}\vZ_{1:t}^{T}+\sigma_{1}^{-2}\vI_{N} = \sigma_{2}^{-2}diag(n_{t}(1),\cdots,n_{t}(N))+\sigma_{1}^{-2}\vI_{N}=diag(\sigma_{1}^{-2}+\sigma_{2}^{-2}n_{t}(1),\cdots,\sigma_{1}^{-2}+\sigma_{2}^{-2}n_{t}(N)).
\end{align*}
Since $\vthe = \vPhi\vgamma+v$, then 
\begin{align*}
    \vthe|\vgamma,\mH_{t+1} \sim \normal(\vPhi\vgamma+\vmu_{**}, diag(\sigma_{1}^{-2}+\sigma_{2}^{-2}n_{t}(1),\cdots,\sigma_{1}^{-2}+\sigma_{2}^{-2}n_{t}(N))^{-1}).
\end{align*}

Therefore, for each item $i \in [N]$,
\begin{equation}
    \begin{split}
    \theta_{i}|\vgamma, \mH_{t+1} &\sim \mathcal{N}(\hat{\mu}_{t+1,\vgamma}(i),\hat{\sigma}_{t+1,\vgamma}^{2}(i)),\\
    \hat{\mu}_{t+1,\vgamma}(i) &= \vx_i^{T}\vgamma+\sigma_{1}^{2}\vZ_{1:t,i}(\sigma_{2}^{2}\vI_{C_{t}}+\sigma_{1}^{2}\vZ_{1:t}^{T}\vZ_{1:t})^{-1}(\vY_{1:t}-\vPhi_{1:t} \vgamma),\\
    \hat{\sigma}_{t+1,\vgamma}^{-2}(i) &= \sigma_{1}^{-2}+\sigma_{2}^{-2}n_{t}(i).
    \end{split}
\end{equation}


\section{Preliminary and Definitions}\label{preliminary}
We first clarify common notations used in our proof. Suppose that there are $N$ items, each with $d$ features. We will recommend a slate of at most $K$ items each time. In total, there are $T$ rounds of the interaction. Let us recall that $\mH_{t} = (A_{l},\vY_{l}(A_{l}))_{l=1}^{t-1}$ includes history up to round $t$ and excluding round $t$, where $\mH_{1} = \emptyset$ and $\vY_{l}(A_{l})=(Y_{k,l},k\in A_{l})$. Given the $\mH_{t}$, the conditional probability is given as $\mathbbm{P}_{t}(\cdot)=\mathbbm{P}(\cdot|\mH_{t})$, and the conditional expectation is given as $\Mean_{t}(\cdot)=\Mean(\cdot|\mH_{t})$. Similarly, we define the probability independent of all history as $\mathbbm{P}(\cdot)$ and the expectation independent of all history as $\Mean(\cdot)$. Additionally, denote the number of pulls of arm $k$ for the first $t$ rounds (including round $t$) as $n_{t}(k)$. Suppose $\boldsymbol{X} \in \mathcal{R}^{d\times d}$, let $\lambda_{1}(\boldsymbol{X})$ denote the maximum eigenvalue of $\boldsymbol{X}$, and $\lambda_{d}(\boldsymbol{X})$ denote the minimum eigenvalue of $\boldsymbol{X}$.

We also need introduce some basic quantities from information theory.  Let $\mathbbm{P}$ and $\mathbbm{Q}$ be two probability measures, and $\mathbbm{P}$ is absolutely continuous with respect to $\mathbbm{Q}$. Then the Kullback–Leibler divergence between $\mathbbm{P}$ and $\mathbbm{Q}$ is defined as $D(\mathbbm{P}\|\mathbbm{Q})=\int log(\frac{d\mathbbm{P}}{d\mathbbm{Q}})d\mathbbm{P}$, where $\frac{d\mathbbm{P}}{d\mathbbm{Q}}$ is the Radon–Nikodym derivative of $\mathbbm{P}$ with respect to $\mathbbm{Q}$. Then the mutual information between two random variables $X$ and $Y$ is defined as the Kullback–Leibler divergence between the joint distribution of $X$ and $Y$ and the product of the marginal distributions, $I(X;Y) = D(\mathbbm{P}(X,Y)\|\mathbbm{P}(X)\mathbbm{P}(Y))$. The mutual information measures the information gained about one random variable by observing the other random variable, which is always non-negative and equals to $0$ only if two random variables are independent to each other. For example, in the proof, we use $I(\vgamma;\mH_{t})$ to quantify the information gain of $\vgamma$ by observing the historic interactions between agents and users, $\mH_{t}$. We also need a conditional mutual information term to quantify the difference between random variables $X$ and $Y$ conditioned on another random variable $Z$, which is defined as $I(X;Y|Z) = \Mean[D(\mathbbm{P}(X,Y|Z)\|\mathbbm{P}(X|Z)\mathbbm{P}(Y|Z))]$ (the expectation is taken over $Z$).

\subsection{General History-Dependent Mutual Information}
Conditional on history $\mH_{t}$, the mutual information between the parameter $\vthe$ and the observations at round $t$, $\vY_{t}$, is defined as follows:
\[
I_{t}(\vthe;A_{t},\vY_{t}) = \Mean_{t}\Big[log\Big(\frac{\mathbbm{P}_{t}(\vthe;A_{t},\vY_{t})}{\mathbbm{P}_{t}(\vthe)\mathbbm{P}_{t}(A_{t},\vY_{t})}\Big)\Big].
\]

Similarly, the history dependent mutual information between the meta parameter $\vgamma$ and the observations at round $t$, $\vY_{t}$, is defined as follows:
\[
I_{t}(\vgamma;A_{t},\vY_{t}) = \Mean_{t}\Big[log\Big(\frac{\mathbbm{P}_{t}(\vgamma,A_{t},\vY_{t})}{\mathbbm{P}_{t}(\vgamma)\mathbbm{P}_{t}(A_{t},\vY_{t})}\Big)\Big].
\]

Then, the history dependent mutual information between the parameters ($\vthe,\vgamma$) and the observations at round $t$, $\vY_{t}$, is defined as below:
\[
I_{t}(\vthe,\vgamma;A_{t},\vY_{t}) = \Mean_{t}\Big[log\Big(\frac{\mathbbm{P}_{t}(\vthe,\vgamma,A_{t},\vY_{t})}{\mathbbm{P}_{t}(\vthe,\vgamma)\mathbbm{P}_{t}(A_{t},\vY_{t})}\Big)\Big].
\]
Finally, the history dependent mutual information between the parameters $\vthe$ and the observations at round $t$, $\vY_{t}$, given that the meta parameter $\vgamma$ is known, is defined as below:
\[
I_{t}(\vthe;A_{t},\vY_{t}|\vgamma) = \Mean_{t}\Big[log\Big(\frac{\mathbbm{P}_{t}(\vthe,A_{t},\vY_{t}|\vgamma)}{\mathbbm{P}_{t}(\vthe|\vgamma)\mathbbm{P}_{t}(A_{t},\vY_{t}|\vgamma)}\Big)\Big].
\]
By the definition of conditional mutual information, we have
\begin{align*}
    I(\cdot;A_{t},\vY_{t}|\mH_{t}) &= \Mean(I_{t}(\cdot;A_{t},\vY_{t})),\\
    I(\cdot;A_{t},\vY_{t}|\vgamma,\mH_{t}) &= \Mean(I_{t}(\cdot;A_{t},\vY_{t}|\vgamma)).
\end{align*}
\subsection{History-Dependent/Independent Mutual Information and Entropy for Semi-Bandits}

Conditional on history $\mH_{t}$, the mutual information between the parameter $\theta_{k}$ and the observations at round $t$, $Y_{k,t}$, is defined as follows:
\[
I_{t}(\theta_{k};k,Y_{k,t}) = \Mean_{t}\Big[log\Big(\frac{\mathbbm{P}_{t}(\theta_{k};k,Y_{k,t})}{\mathbbm{P}_{t}(\theta_{k})\mathbbm{P}_{t}(k,Y_{k,t})}\Big)\Big].
\]

Similarly, the history dependent mutual information between the meta parameter $\vgamma$ and the observations at round $t$, $Y_{k,t}$, is defined as follows:
\[
I_{t}(\vgamma;k,Y_{k,t}) = \Mean_{t}\Big[log\Big(\frac{\mathbbm{P}_{t}(\vgamma,k,Y_{k,t})}{\mathbbm{P}_{t}(\vgamma)\mathbbm{P}_{t}(k,Y_{k,t})}\Big)\Big].
\]

Then, the history dependent mutual information between the parameters ($\theta_{k},\vgamma$) and the observations at round $t$, $Y_{k,t}$, is defined as below:
\[
I_{t}(\theta_{k},\vgamma;k,Y_{k,t}) = \Mean_{t}\Big[log\Big(\frac{\mathbbm{P}_{t}(\theta_{k},\vgamma,k,Y_{k,t})}{\mathbbm{P}_{t}(\theta_{k},\vgamma)\mathbbm{P}_{t}(k,Y_{k,t})}\Big)\Big].
\]

Finally, the history dependent mutual information between the parameters $\theta_{k}$ and the observations at round $t$, $Y_{k,t}$, given that the meta parameter $\vgamma$ is known, is defined as below:
\[
I_{t}(\theta_{k};k,Y_{k,t}|\vgamma) = \Mean_{t}\Big[log\Big(\frac{\mathbbm{P}_{t}(\theta_{k},k,Y_{k,t}|\vgamma)}{\mathbbm{P}_{t}(\theta_{k}|\vgamma)\mathbbm{P}_{t}(k,Y_{k,t}|\vgamma)}\Big)\Big].
\]

Based on the definition of entropy, we further defined the history dependent entropy terms as follows:
\begin{align*}
    \text{\textbf{Conditional Entropy of $\theta_{k}$}}&: h_{t}(\theta_{k}) = -\Mean_{t}[log(\mathbbm{P}_{t}(\theta_{k}))],\\
    \text{\textbf{Conditional Entropy of $\vgamma$}}&: h_{t}(\vgamma) = -\Mean_{t}[log(\mathbbm{P}_{t}(\vgamma))],\\
    \text{\textbf{Conditional Entropy of $\theta_{k}$ given $\vgamma$}}&: h_{t}(\theta_{k}|\vgamma) = -\Mean_{t}[log(\mathbbm{P}_{t}(\theta_{k}|\vgamma))].
\end{align*}

Straightforwardly, by the definition of conditional mutual information, the history independent conditional mutual information terms are defined as the expectation of the history dependent term. 
\begin{align*}
    I(\cdot;k,Y_{k,t}|\mH_{t}) &= \Mean(I_{t}(\cdot;k,Y_{k,t})),\\
    I(\cdot;k,Y_{k,t}|\vgamma,\mH_{t}) &= \Mean(I_{t}(\cdot;k,Y_{k,t}|\vgamma)).
\end{align*}

Similarly, the history independent conditional entropy terms are defined as follows:
\begin{align*}
    h(\cdot|\mH_{t}) &= \Mean(h_{t}(\cdot)), \\ h(\cdot|\vgamma,\mH_{t}) &= \Mean(h_{t}(\cdot|\vgamma)).
\end{align*}

\subsection{Others}
In the following, we restate several properties of the mutual information and entropy and an inequality lemma that we mainly used in our proof.

\textbf{Decomposition of Mutual Information.} Based on the definition of mutual information and entropy, we can decompose the mutual information term as below.
\begin{align*}
    I_{t}(\cdot;k,Y_{k,t}) &= h_{t}(\cdot)-h_{t+1}(\cdot),\\
    I_{t}(\cdot;k,Y_{k,t}|\vgamma) &= h_{t}(\cdot|\vgamma)-h_{t+1}(\cdot|\vgamma).
\end{align*}

\textbf{Chain Rule.} $I(X,Y;Z)=I(Y;Z)+I(X;Z|Y)$.

\textbf{Weyl's inequality.} For Hermitian matrix $\boldsymbol{A}, \boldsymbol{B} \in \mathbb{C}^{d\times d}$ and $i=1,\cdots,d,$ $\lambda_{i}(\boldsymbol{A}+\boldsymbol{B})\leq \lambda_{i}(\boldsymbol{A}) + \lambda_{1}(\boldsymbol{B})$.

\section{Main Proof}
\subsection{Proof for Theorem \ref{general regret bound}}

\begin{proof} First, following the property of mutual information and the chain rule of conditional mutual information, we can derive that $I_{t}(\vthe;A_{t},\vY_{t})\leq I_{t}(\vthe,\vgamma;A_{t},\vY_{t})=I_{t}(\vgamma;A_{t},\vY_{t})+I_{t}(\vthe;A_{t},\vY_{t}|\vgamma)$. Taking the square root of it and applying the Cauchy-Schwartz inequality, we have that $\sqrt{I_{t}(\vgamma;A_{t},\vY_{t})+I_{t}(\vthe;A_{t},\vY_{t}|\vgamma)}\leq \sqrt{I_{t}(\vgamma;A_{t},\vY_{t})}+\sqrt{I_{t}(\vthe;A_{t},\vY_{t}|\vgamma)}$. After that, using the assumption that $\Gamma_{t}\leq \Gamma$ $w.p. 1$ and collecting the terms, we finish the proof. Here, the regret bound can be divided into two parts, where the first part is the cost of learning the meta parameter $\vgamma$, the second part is the regret for learning $\vthe$ with known $\vgamma$.

Mathematically,
\begin{align*}
   BR(T)& = \Mean[\sum_{t}\Delta_{t}] \\
   &\le \Mean[\sum_{t} \Gamma_{t}\sqrt{I_{t}(\vthe;A_{t},\vY_{t})}+\epsilon_{t}]\\
   &\le \Mean[\sum_{t} \Gamma_{t}\sqrt{I_{t}(\vthe,\vgamma;A_{t},\vY_{t})}]+\Mean[\sum_{t}\epsilon_{t}]\\
   &= \Mean[\sum_{t} \Gamma_{t}\sqrt{I_{t}(\vgamma;A_{t},\vY_{t})+I_{t}(\vthe;A_{t},\vY_{t}|\vgamma)}]+\Mean[\sum_{t}\epsilon_{t}]\\
   &\le \Mean[\Gamma_{t}\sum_{t}\sqrt{I_{t}(\vgamma;A_{t},\vY_{t})}+\sqrt{I_{t}(\vthe;A_{t},\vY_{t}|\vgamma)}]+\Mean[\sum_{t}\epsilon_{t}]\\
   &\leq \underbrace{\Gamma\sum_{t}\Mean[\sqrt{I_{t}(\vgamma;A_{t},\vY_{t})}]}_\text{Regret due to not knowing $\vgamma$}
    {+} \underbrace{\sum_{t}\Gamma\Mean[\sqrt{I_{t}(\vthe;A_{t},\vY_{t}|\vgamma)}]{+}\Mean[\epsilon_{t}]}_\text{Regret suffered even with known $\gamma$}.
\end{align*}
The first inequality directly uses the (\ref{eqn:Delta_decomp}). The second inequality follows the property of mutual information that $I(X;Z)\leq I(X,Y;Z)$. Here $X=\vthe$, $Y = \vgamma$, and $Z=(A_{t},\vY_{t})$. The third equality uses the chain rule of mutual information, $I(X,Y;Z)=I(Y;Z)+I(X;Z|Y)$. The forth inequality follows the fact that $\sqrt{a+b}\leq \sqrt{a}+\sqrt{b}$. The final inequality follows that $\Gamma_{t}\leq \Gamma$ $w.p. 1$.
\end{proof}

\subsection{Proof for Theorem \ref{Theorem1}}

\textbf{Roadmap:} \textit{There are two main steps in the proof. First, we decompose the Bayes regret into two parts as (\ref{decomp BR semi}). To derive the Bayes regret decomposition, we first show that (\ref{per_round_regret}) holds for all $t\in[T]$ in Lemma \ref{expected round regret}, and then prove that (\ref{decomp BR semi}) holds under the condition (\ref{per_round_regret}) in Lemma \ref{regret bound for semi bandit}. Second, we get the bound of each component in (\ref{decomp BR semi}). In particular, the upper bounds of $\Gamma_{t}$ and $\epsilon_{t}$ are derived in Lemma \ref{expected round regret}, whereas the upper bounds of $I(\vgamma;\mH_{T+1})$ and $I(\theta_{k};\mH_{T+1})$ are derived in Lemma \ref{bound for MI}. Gathering the bounds of all components, we get the regret bound in Theorem \ref{Theorem1}. Following are the details of the main proof.}

We start by stating several lemmas, which will be used in our main proof. Proofs of the lemmas are deferred to Appendix \ref{proof of lemma}. Without loss of generality, we assume that all available items have bounded norm (Assumption 1) and all parameters are bounded (Assumption 2).

Using the independence between rewards generated by different arms, we first decompose the per-round expected regret in a similar form of (\ref{eqn:Delta_decomp}), with suitably selected history-dependent constants $\Gamma_{t}$ and $\epsilon_{t}$. Based on the properties of the Gaussian distributions and the fact that MTSS samples rewards from corresponding posterior distributions for every round, we bound both $\Gamma_{t}$ and $\epsilon_{t}$ by functions of $\frac{\delta}{N} \in (0,1]$.

\begin{lemma} \label{expected round regret}
 For any $\mH_{t}$-adapted sequence of actions $(A_{l})_{l=1}^{t-1}$, and any $\delta$ such that $\frac{\delta}{N} \in (0,1]$, the expected regret in round t conditioned on $\mH_{t}$ is bounded as
 \begin{equation}
    \Mean_{t}[\Delta_{t}] \leq \sum_{k \in [N]} \mathbb{P}_{t}(k\in A_{t}) \Gamma_{k,t}\sqrt{I_{t}(\theta_{k};k,Y_{k,t})}+\epsilon_{t},\label{per_round_regret}
\end{equation}
where
\begin{align*}
    \Gamma_{k,t} = 4\sqrt{\frac{\hat{\sigma}_{t}^{2}(k)}{log(1+\hat{\sigma}_{t}^{2}(k)/\sigma_{2}^{2})} log(\frac{4N}{\delta})}, && \epsilon_{t} = \sum_{k \in [N]} \mathbb{P}_{t}(k\in A_{t})\sqrt{2\delta\frac{1}{N}\hat{\sigma}^{2}_{t}(k)}.
\end{align*}
Moreover, for each k, the following history-independent bound holds almost surely.
\begin{equation*}
    \hat{\sigma}_{t}^{2}(k)\leq \lambda_{1}(\Sigma_{\vgamma})+\sigma_{1}^{2}.
\end{equation*}
\end{lemma}

Based on \textbf{Lemma \ref{expected round regret}}, we get that $\Gamma_{k,t}=O(\sqrt{log(\frac{N}{\delta})})$ and $\epsilon_{t}=O(K\sqrt{\frac{\delta}{N}})$. Then, similar to \textbf{Theorem \ref{general regret bound}}, based on the per-round conditional expected regret decomposition, we develop a decomposition of the total regret over $T$ rounds of interactions by summing the per-round regret over $T$ rounds and then taking the expectation over historical interactions.  

\begin{lemma} \label{regret bound for semi bandit} 
Suppose that (\ref{per_round_regret}) holds for all $t\in [T]$, for some suitably chosen $\Gamma_{k,t}$ and $\epsilon_{t}$. Let $\Gamma_{k}$ and $\Gamma$ be some non-negative constants such that $\Gamma_{k,t}\leq \Gamma_{k} \leq \Gamma$ holds for all $t \in [T]$ and $k \in [N]$ almost surely. Then
\begin{equation} \label{decomp BR semi}
    BR(T)\leq 
    \underbrace{\Gamma K \sqrt{TI(\vgamma;\mH_{T+1})}}_\text{Regret due to not knowing $\vgamma$}
    {+} \underbrace{\Gamma\sqrt{NTK}\sqrt{\frac{1}{N}\sum_{k \in [N]}
    I(\theta_{k};\mH_{T+1}|\vgamma)}+\sum_{t}\Mean[\epsilon_{t}]}_\text{Regret suffered even with known $\gamma$}.
\end{equation}
\end{lemma}

Here, the first term is the cost for learning the meta parameter $\vgamma$, and the second term is regret for unknown item-specific parameter $\theta_{k}$ given known $\vgamma$. We show the benefits of information sharing among items mainly by the first term, which indicates that the extra regret due to unknown $\vgamma$ is much lower that the cost of learning $\vthe$ with known $\vgamma$. 
Using the assumption that $\Gamma_{k,t}\leq \Gamma_{k} \leq \Gamma$ $w.p.1$ and the bound of $\Gamma_{k,t}$ and $\epsilon_{t}$, we directly get the bound of  $\Gamma$ and $\Mean[\epsilon_{t}]$. Then, our next lemma find the bound of the mutual information terms involved in (\ref{decomp BR semi}), by using the properties of Gaussian distribution and the properties of LMM.  

\begin{lemma} \label{bound for MI}
For any $k\in [N]$ and any $\mH_{T+1}$-adapted sequence of actions $(A_{l})_{l=1}^{T}$, we have
\begin{align*}
    I(\vgamma;\mH_{T+1})\leq \frac{d}{2}log\Big(1+\frac{N\lambda_{1}(\Sigma_{\vgamma})}{\sigma_{1}^{2}+\sigma_{2}^{2}/T}\Big), && I(\theta_{k};\mH_{T+1}|\vgamma)\leq  \frac{1}{2}log(1+\frac{\sigma_{1}^{2}}{\sigma_{2}^{2}}T).
\end{align*}
\end{lemma} 

Now we are ready to combine these results and present our main proof of \textbf{Theorem \ref{Theorem1}}. Specifically, we get the bounds of $\Gamma_{k,t}$ and $\epsilon_{t}$ from \textbf{Lemma \ref{expected round regret}} and the bounds of mutual information terms from \textbf{Lemma \ref{bound for MI}}, and then plug them into the regret decomposition derived in \textbf{Lemma \ref{regret bound for semi bandit}}.



\begin{proof}[Proof of Theorem \ref{Theorem1}]

From \textbf{Lemma \ref{expected round regret}}, we showed that (\ref{per_round_regret}) holds for suitably chosen $\Gamma_{k,t}$ and $\epsilon_{t}$. Using the upper bounds of $\hat{\sigma}_{t}(k)$ in \textbf{Lemma \ref{expected round regret}}, since $\sqrt{\frac{x}{log(1+ax)}}$ is an increasing function in x when $a>0$, we can bound $w.p. 1$ that
\begin{equation*}
    \Gamma_{k,t} \leq 4\sqrt{\frac{\lambda_{1}(\Sigma_{\vgamma})+\sigma_{1}^{2}}{log(1+(\lambda_{1}(\Sigma_{\vgamma})+\sigma_{1}^{2})/\sigma_{2}^{2})} log(\frac{4N}{\delta})}=\Gamma.
\end{equation*}
Then, we have the upper bound of $\Gamma_{k,t} \leq \Gamma$ for all $t$ and $k$ $w.p. 1$. 
Similarly, we have
\begin{equation*}
    \epsilon_{t}  \leq \sum_{k \in [N]} \mathbb{P}_{t}(k\in A_{t})\sqrt{2\delta\frac{1}{N}(\lambda_{1}(\Sigma_{\vgamma})+\sigma_{1}^{2})}.
\end{equation*}
From \textbf{Lemma \ref{regret bound for semi bandit}}, for any $\delta>0$, let $c_{1}= 4\sqrt{\frac{\lambda_{1}(\Sigma_{\vgamma})+\sigma_{1}^{2}}{log(1+(\lambda_{1}(\Sigma_{\vgamma})+\sigma_{1}^{2})/\sigma_{2}^{2})} log(\frac{4N}{\delta})}$, we have
\begin{align*}
    BR(T)&\leq \Gamma K \sqrt{TI(\vgamma;\mH_{T+1})}+\Gamma\sqrt{NTK}\sqrt{\frac{1}{N}\sum_{k \in [N]}
    I(\theta_{k};\mH_{T+1}|\vgamma)}+\sum_{t}\Mean(\epsilon_{t})\\
    &\leq c_{1} K\sqrt{T}\sqrt{\frac{d}{2}log\Big(1+\frac{N\lambda_{1}(\Sigma_{\vgamma})}{\sigma_{1}^{2}+\sigma_{2}^{2}/T}\Big)} 
    +c_{1}\sqrt{NTK}\sqrt{\frac{1}{2}log(1+\frac{\sigma_{1}^{2}}{\sigma_{2}}T)}+TK\sqrt{2\delta\frac{1}{N}(\lambda_{1}(\Sigma_{\vgamma})+\sigma_{1}^{2})}
\end{align*}
The inequality holds by first using the upper bound of mutual information in \textbf{Lemma \ref{bound for MI}}, and the upper bound of $\Gamma$ and $\epsilon_{t}$, then we derive the history-independent upper bound of $\sum_{t}\Mean(\epsilon_{t})$ as the following.
\begin{align*}
    &\sum_{t}\Mean\Big[\sum_{k \in [N]} \mathbb{P}_{t}(k\in A_{t})\sqrt{2\delta\frac{1}{N}(\lambda_{1}(\Sigma_{\vgamma})+\sigma_{1}^{2})}\Big]\\
    &=\sqrt{2\delta\frac{1}{N}(\lambda_{1}(\Sigma_{\vgamma})+\sigma_{1}^{2})}\Mean\Big[\sum_{t}\sum_{k \in [N]} \mathbb{P}_{t}(k\in A_{t})\Big]\\
    &\leq TK\sqrt{2\delta\frac{1}{N}(\lambda_{1}(\Sigma_{\vgamma})+\sigma_{1}^{2})}.
\end{align*}
Let $\delta = \frac{1}{T^{2}}$, 
\begin{align*}
    BR(T)&\leq c_{1} K\sqrt{T}\sqrt{\frac{d}{2}log\Big(1+\frac{N\lambda_{1}(\Sigma_{\vgamma})}{\sigma_{1}^{2}+\sigma_{2}^{2}/T}\Big)} 
    +c_{1}\sqrt{NTK}\sqrt{\frac{1}{2}log(1+\frac{\sigma_{1}^{2}}{\sigma_{2}}T)}+K\sqrt{\frac{2}{N}(\lambda_{1}(\Sigma_{\vgamma})+\sigma_{1}^{2})}\\
    &= O(K\sqrt{Tdlog(N)log(NT^{2})}+\sqrt{NTKlog(T)log(NT^{2})}+K\sqrt{\frac{1}{N}})\\
    &=\tilde{O}(K\sqrt{Td}+\sqrt{NTK}).
\end{align*}

\end{proof}


\section{Proof of Lemmas} \label{proof of lemma}

\subsection{Proof for Lemma \ref{expected round regret}}


\begin{proof} First, using the probability matching property of Thompson Sampling and the independence between the rewards generated by different arms, we decompose the per-round expected regret as $\sum_{k \in [N]} \mathbb{P}_{t}(k\in A_{t}) \Mean_{t}[\hat{\theta}_{k,t}-\theta_{k}]$, where $\hat{\theta}_{k,t}$ is the estimated mean reward for arm $k$ given the history $\mH_{t}$. Then, following Lemma 5 in \citet{lu2019information}, we define a confidence set $\Theta_{t}(k)$ for both $\hat{\theta}_{k,t}$ and $\theta_{k}$ with high probability for each arm $k$ at round $t$, with suitably selected non-negative random variables $\Gamma_{k,t}$, which leads to the bound of $\Mean_{t}[\hat{\theta}_{k,t}-\theta_{k}]$ and concludes the proof of the first part of the lemma directly. The $\epsilon_{t}$ is some non-negative random variables derived appropriately. For the second part of the lemma, we bound the $\Gamma_{k,t}$ and $\epsilon_{t}$ by finding the upper bound of $\hat{\sigma}_{t}^{2}(k)$ for each arm $k$ at round $t$ conditional on the history $\mH_{t}$. 

Now we are ready to prove \textbf{Lemma \ref{expected round regret}} in detail, as follows.\\

Since $\sum_{k\in A_{*}}\theta_{k}|\mH_{t} \stackrel{d}{=} \sum_{k\in A_{t}}\hat{\theta}_{k,t}|\mH_{t}$, we have
\begin{align*}
    \Mean_{t}[\Delta_{t}] &= \Mean_{t}[\sum_{k \in A_{*}}\theta_{k}-\sum_{k \in A_{t}}\theta_{k}]\\
    & = \Mean_{t}[\sum_{k \in A_{t}}\hat{\theta}_{k,t}-\sum_{k \in A_{t}}\theta_{k}]\\
    &= \Mean_{t}[\sum_{k \in [N]}\mathbbm{1}(k \in A_{t})(\hat{\theta}_{k,t}-\theta_{k})]\\
    & = \sum_{k \in [N]} \mathbb{P}_{t}(k\in A_{t}) \Mean_{t}[\hat{\theta}_{k,t}-\theta_{k}].
\end{align*}


For each $k\in [N]$, we know that $\hat{\theta}_{k,t}|\mH_{t}\sim  \mathcal{N}(\hat{\mu}_{t}(k),\hat{\sigma}_{t}^{2}(k))$. Let us consider the confidence set of $\hat{\theta}_{k}$ for each arm $k$:
\[
\Theta_{t}(k) = \{\theta:|\theta-\hat{\mu}_{t}(k)|\leq\frac{\Gamma_{k,t}}{2}\sqrt{I_{t}(\theta_{k};k,Y_{k,t})}\}.
\]
The history dependent conditional mutual entropy of $\theta_{k}$ given the history $\mH_{t}$, $I_{t}(\theta_{k};k,Y_{k,t})$, can be computed as follows:
\begin{align*}
    I_{t}(\theta_{k};k,Y_{k,t}) &=  h_{t}(\theta_{k})-h_{t+1}(\theta_{k})\\
    &=\frac{1}{2}log(det(2\pi e\hat{\sigma}_{t}^{2}(k))-\frac{1}{2}log(det(2\pi e\hat{\sigma}_{t+1}^{2}(k))\\
    &=\frac{1}{2}log(\hat{\sigma}_{t}^{2}(k)\hat{\sigma}_{t+1}^{-2}(k))\\
    &=\frac{1}{2}log(\hat{\sigma}_{t}^{2}(k)[\hat{\sigma}_{t}^{-2}(k)+\sigma_{2}^{-2}])\\
    &= \frac{1}{2}log\Big(1+\frac{\hat{\sigma}_{t}^{2}(k)}{\sigma_{2}^{2}}\Big).
\end{align*}

For $\frac{\delta}{N} > 0$, let
\[
\Gamma_{k,t} = 4\sqrt{\frac{\hat{\sigma}_{t}^{2}(k)}{log(1+\hat{\sigma}_{t}^{2}(k)/\sigma_{2}^{2})} log(\frac{4N}{\delta})}.
\]

Then, following Lemma 5 in \citet{lu2019information}, for any $k$
and any $\delta$ such that $\frac{\delta}{N} \in (0,1]$, we have
\[
\mathbb{P}_{t}(\hat{\theta}_{k,t}\in \Theta_{t}(k))\geq1-\frac{\delta}{2N}.
\]

Now we continue the regret decomposition as
\begin{align*}
    \Mean_{t}[\hat{\theta}_{k,t}-\theta_{k}] &=  \Mean_{t}[\mathbbm{1}(\hat{\theta}_{k,t},\theta_{k}\in \Theta_{t}(k))(\hat{\theta}_{k,t}-\theta_{k})]+ \Mean_{t}[\mathbbm{1}^{c}(\hat{\theta}_{k,t},\theta_{k}\in \Theta_{t}(k))(\hat{\theta}_{k,t}-\theta_{k})]\\
    &\leq \Gamma_{k,t}\sqrt{I_{t}(\theta_{k};k,Y_{k,t})}+ \sqrt{\mathbb{P}(\hat{\theta}_{k,t}\text{ or }\theta_{k}\not\in \Theta_{t}(k))\Mean_{t}[(\hat{\theta}_{k,t}-\theta_{k})^{2}]}\\
    &\leq \Gamma_{k,t}\sqrt{I_{t}(\theta_{k};k,Y_{k,t})}+ \sqrt{\delta\frac{1}{N}\Mean_{t}[(\hat{\theta}_{k,t}-\hat{\mu}_{t}(k))^{2}+(\theta_{k}-\hat{\mu}_{t}(k))^{2}]}\\
    &\leq \Gamma_{k,t}\sqrt{I_{t}(\theta_{k};k,Y_{k,t})}+\sqrt{2\delta\frac{1}{N}\hat{\sigma}^{2}_{t}(k)}.
\end{align*}
The second inequality uses that $\mathbb{P}(\hat{\theta}_{k,t}\text{ or }\theta_{k}\not\in \Theta_{t}(k))\leq \mathbb{P}(\hat{\theta}_{k,t}\not\in \Theta_{t}(k))+\mathbb{P}(\theta_{k}\not\in \Theta_{t}(k))=\frac{\delta}{N}$.
Therefore, we conclude our proof of the first part.
\begin{align*}
    \Mean_{t}[\Delta_{t}] & = \sum_{k \in [N]} \mathbb{P}_{t}(k\in A_{t}) \Mean_{t}[\hat{\theta}_{k,t}-\theta_{k}]\\
    &\leq \sum_{k \in [N]} \mathbb{P}_{t}(k\in A_{t}) \Big(\Gamma_{k,t}\sqrt{I_{t}(\theta_{k};k,Y_{k,t})}+\sqrt{2\delta\frac{1}{N}\hat{\sigma}^{2}_{t}(k)}\Big)\\
    &= \sum_{k \in [N]} \mathbb{P}_{t}(k\in A_{t}) \Gamma_{k,t}\sqrt{I_{t}(\theta_{k};k,Y_{k,t})}+\sum_{k \in [N]} \mathbb{P}_{t}(k\in A_{t})\sqrt{2\delta\frac{1}{N}\hat{\sigma}^{2}_{t}(k)}\\
    & = \sum_{k \in [N]} \mathbb{P}_{t}(k\in A_{t}) \Gamma_{k,t}\sqrt{I_{t}(\theta_{k};k,Y_{k,t})}+\epsilon_{t}.
\end{align*}

\textbf{Bounding $\hat{\sigma}_{t}^{2}(k)$.} Recall that
\begin{align*}
    \hat{\sigma}_{t}^{2}(k)&=\vx_i^T{\Cov}_{\vgamma}\vx_i+\sigma_{1}^{2}-(\vx_i^T{\Cov}_{\vgamma}\vPhi^{T}_{1:t}+\sigma_{1}^{2}\vZ_{1:t,i})(\sigma_{2}^{2}\vI_{Kt}+\vPhi_{1:t}{\Cov}_{\vgamma}\vPhi^{T}_{1:t}+\sigma_{1}^{2}\vZ_{1:t}^{T}\vZ_{1:t})^{-1}(\vx_i^T{\Cov}_{\vgamma}\vPhi^{T}_{1:t}+\sigma_{1}^{2}\vZ_{1:t,i})^{T}\\
    &\leq \vx_k^T\Sigma_{\vgamma}\vx_k+\sigma_{1}^{2}\\
    &\leq max_{k \in [N]} \vx_k^T\Sigma_{\vgamma}\vx_k+\sigma_{1}^{2}\\
    &\leq max_{k \in [N]} \vx_k^T\lambda_{1}(\Sigma_{\vgamma})\vx_k+\sigma_{1}^{2}\\
    &\leq \lambda_{1}(\Sigma_{\vgamma})+\sigma_{1}^{2}.
\end{align*}
The first inequality results form that $(\vx_i^T{\Cov}_{\vgamma}\vPhi^{T}_{1:t}+\sigma_{1}^{2}\vZ_{1:t,i})(\sigma_{2}^{2}\vI_{Kt}+\vPhi_{1:t}{\Cov}_{\vgamma}\vPhi^{T}_{1:t}+\sigma_{1}^{2}\vZ_{1:t}^{T}\vZ_{1:t})^{-1}(\vx_i^T{\Cov}_{\vgamma}\vPhi^{T}_{1:t}+\sigma_{1}^{2}\vZ_{1:t,i})^{T}$ is positive semi-definite. By the assumption that $\|x_{i}\|_{2}\leq 1$, we have the last inequality.
\end{proof}

\subsection{Proof for Lemma \ref{regret bound for semi bandit}}

\begin{proof} First, suppose that (\ref{per_round_regret}) holds for all round $t$, we adapt \textbf{Theorem \ref{general regret bound}} to bound the expected regret for each arm $k$. Here, $A_{t} = k$, $\vY_{t} = Y_{k,t}$, and $\Gamma_{t} = \mathbb{P}_{t}(k\in A_{t}) \Gamma_{k,t}$. Summing the regret bound for each arm $k$, similar to \textbf{Theorem \ref{general regret bound}}, we can decompose the Bayes regret bound into three parts, where the first part is the cost of learning $\vgamma$, and the rest two parts constitute the cost of learning $\vthe=(\theta_{k})_{k=1}^{N}$. While the third part is bounded in \textbf{Lemma \ref{expected round regret}}, we bound the first two parts separately. Particularly, by \textbf{Lemma \ref{decomp for MI}}, the expectation of the history dependent mutual information terms in the first two parts are bounded by the history independent mutual information, respectively. The proof is concluded by utilizing inequalities and the assumption that $\Gamma_{k,t}\leq \Gamma_{k} \leq \Gamma$ $w.p. 1$.

Mathematically,
\begin{align*}
    BR(T) &= \Mean[R(T)] = \Mean[\sum_{t}\Delta_{t}]\\
    & \leq \Mean\Big[\sum_{t}\sum_{k \in [N]} \mathbb{P}_{t}(k\in A_{t}) \Gamma_{k,t}\sqrt{I_{t}(\theta_{k};k,Y_{k,t})}\Big]+\Mean\sum_{t}\epsilon_{t}\\
    & = \sum_{k \in [N]}\Mean\Big[\sum_{t} \mathbb{P}_{t}(k\in A_{t}) \Gamma_{k,t}\sqrt{I_{t}(\theta_{k};k,Y_{k,t})}\Big]+\Mean\sum_{t}\epsilon_{t}\\
     &\leq \underbrace{\sum_{k\in [N]}\Gamma_{k} \sum_{t} \Mean [\mathbb{P}_{t}(k\in A_{t}) \sqrt{I_{t}(\vgamma;k,Y_{k,t})}]}_\text{Regret due to not knowing $\vgamma$}
    {+} \underbrace{\sum_{k \in [N]} \Gamma_{k} \sum_{t}\Mean[\mathbb{P}_{t}(k\in A_{t}) \sqrt{I_{t}(\theta_{k};k,Y_{k,t}|\vgamma)}]+\sum_{t}\Mean[\epsilon_{t}]}_\text{Regret suffered even with known $\gamma$}.
\end{align*}
The first inequality directly uses the (\ref{per_round_regret}). Similar to the proof of \textbf{Theorem \ref{general regret bound}}, we have the second inequality with the assumption that $\Gamma_{k,t}\leq \Gamma_{k}$ $w.p. 1$.

We now derive the bounds of the first two terms separately. For the first term, we have
\begin{align*}
    &\sum_{k\in [N]}\Gamma_{k} \Mean\sum_{t} \mathbb{P}_{t}(k\in A_{t}) \sqrt{I_{t}(\vgamma;k,Y_{k,t})}\\
    &\leq \Gamma \Mean\sum_{t} \sum_{k \in [N]} \mathbb{P}_{t}(k\in A_{t})\sqrt{I_{t}(\vgamma;k,Y_{k,t})}\\
    & \leq \Gamma K \Mean\sum_{t}\sqrt{I_{t}(\vgamma;A_{t},\vY_{t})}\\
&\leq  \Gamma K E\sqrt{T\sum_{t}I_{t}(\vgamma;A_{t},\vY_{t})}\\
&\leq  \Gamma K \sqrt{T\sum_{t}\Mean I_{t}(\vgamma;A_{t},\vY_{t})}\\
&=  \Gamma K \sqrt{TI(\vgamma;\mH_{T+1})}.
\end{align*}
The first inequality follows that $\Gamma_{k}\leq \Gamma$ $w.p. 1$. The second inequality holds by \textbf{Lemma \ref{decomp for MI}}. The third inequality follows $\sum_{i}a_{i}b_{i} \leq \sqrt{\sum_{i}a_{i}^{2}\sum_{i}b_{i}^{2}}$ with $a_{i} = 1$ and $b_{i} = \sqrt{I_{t}(\vgamma;A_{t},Y_{t})}$. The forth inequality follows the Jensen's Inequality for $\sqrt{\cdot}$. The final equality follows the chain rule of mutual information. Specifically, $I(\vgamma;\mH_{T+1})=\sum_{t=1}^{T}I(\vgamma;A_{t},\vY_{t}|\mH_{t})=\Mean\sum_{t=1}^{T}I_{t}(\vgamma;A_{t},\vY_{t})$.

For the second term, we have
\begin{align*}
    &\sum_{k \in [N]} \Gamma_{k}  \Mean\sum_{t}\mathbb{P}_{t}(k\in A_{t}) \sqrt{I_{t}(\theta_{k};k,Y_{k,t}|\vgamma)}\\
    &= \sum_{k \in [N]} \Gamma_{k}  \Mean\Big[\sum_{t}\sqrt{\mathbb{P}_{t}(k\in A_{t})} \sqrt{\mathbb{P}_{t}(k\in A_{t})I_{t}(\theta_{k};k,Y_{k,t}|\vgamma)}\Big]\\
    &\leq \sum_{k \in [N]} \Gamma_{k}  \Mean\Big[\sqrt{\sum_{t}\mathbb{P}_{t}(k\in A_{t}) \sum_{t}\mathbb{P}_{t}(k\in A_{t})I_{t}(\theta_{k};k,Y_{k,t}|\vgamma)}\Big]\\
     &\leq \sum_{k \in [N]} \Gamma_{k}  \sqrt{\Mean\sum_{t}\mathbb{P}_{t}(k\in A_{t})} \sqrt{\Mean\sum_{t} \mathbb{P}_{t}(k\in A_{t})I_{t}(\theta_{k};k,Y_{k,t}|\vgamma)}\\
    &\leq \sum_{k \in [N]} \Gamma_{k}  \sqrt{\Mean[n_{T}(k)]} \sqrt{I(\theta_{k};\mH_{T+1}|\vgamma)}\\
     &\leq \sqrt{N\sum_{k\in [N]}\Mean[n_{T}(k)]}\sqrt{\frac{1}{N}\sum_{k \in [N]} \Gamma_{k}^{2} I(\theta_{k};\mH_{T+1}|\vgamma)}\\
    & \leq \sqrt{NTK}\sqrt{\frac{1}{N}\sum_{k \in [N]} \Gamma^{2} I(\theta_{k};\mH_{T+1}|\vgamma)}\\
    & = \Gamma \sqrt{NTK}\sqrt{\frac{1}{N}\sum_{k \in [N]} I(\theta_{k};\mH_{T+1}|\vgamma)}.
\end{align*}
The first inequality uses the Cauchy-Schwartz inequality, that $\sum_{i}a_{i}b_{i} \leq \sqrt{\sum_{i}a_{i}^{2}\sum_{i}b_{i}^{2}}$. Here $a_{i}=\sqrt{\mathbb{P}_{t}(k\in A_{t})}$ and $b_{i}=\sqrt{\mathbb{P}_{t}(k\in A_{t})I_{t}(\theta_{k};k,Y_{k,t}|\vgamma)}$. The second inequality follows that $\Mean(XY)\leq\sqrt{\Mean(X^{2})\Mean(Y^{2})}$ for $X,Y >0$ $w.p. 1$. The third inequality uses the result of \textbf{Lemma \ref{decomp for MI}}. The next inequality uses the Cauchy-Schwartz inequality again with $a_{i}=\sqrt{\Mean[n_{T}(k)]}$ and $b_{i}=\sqrt{I(\theta_{k};\mH_{T+1}|\vgamma)}$. The last inequality is because of $\Gamma_{k}\leq \Gamma$ $w.p. 1$. 
\end{proof}

\subsection{Proof for Lemma \ref{bound for MI}}

\begin{proof}
First, we derive the mutual information of the meta-parameter $\vgamma$ given the history as follows.
\begin{align*}
    I(\vgamma;\mH_{T+1}) &= h(\vgamma)-h(\vgamma|\mH_{T+1})\\
    &=h(\vgamma)-\Mean[h_{T+1}(\vgamma)]\\
    &=\frac{1}{2}log(det(2\pi e\Sigma_{\vgamma}))-\Mean[\frac{1}{2}log(det(2\pi e \tilde{\Sigma}_{T+1}))]\\
    &=\frac{1}{2}\Mean[log(det(\Sigma_{\vgamma})det(\tilde{\Sigma}_{T+1}^{-1}))]\\
    &=\frac{1}{2}\Mean[log(\prod_{i=1}^{d}\lambda_{i}(\Sigma_{\vgamma})\lambda_{i}(\tilde{\Sigma}_{T+1}^{-1}))]\\
    &\leq \frac{1}{2}\Mean\Big[log(\prod_{i=1}^{d}\lambda_{i}(\Sigma_{\vgamma})\Big(\frac{1}{\lambda_{i}(\Sigma_{\vgamma})}+\frac{N}{\sigma_{2}^{2}/T+\sigma_{1}^{2}}\Big)\Big]\\
    &\leq \frac{d}{2}log\Big(1+\frac{N\lambda_{1}(\Sigma_{\vgamma})}{\sigma_{1}^{2}+\sigma_{2}^{2}/T}\Big).
\end{align*}

For the final inequality, we derive the history independent bound as follows.
\begin{align*}
    \lambda_{i}(\tilde{{\Cov}}_{T+1}^{-1})) &= \lambda_{i}({\Cov}_{\vgamma}^{-1}+\sum_{i=1}^{N}\frac{n_{T}(i)}{\sigma_{2}^{2}+\sigma_{1}^{2}n_{T}(i)}\vx_{i}\vx_{i}^{T})\\
    &\leq \lambda_{i}({\Cov}_{\vgamma}^{-1})+\lambda_{1}(\sum_{i=1}^{N}\frac{n_{T}(i)}{\sigma_{2}^{2}+\sigma_{1}^{2}n_{T}(i)}\vx_{i}\vx_{i}^{T})\\
    &\leq \frac{1}{\lambda_{i}({\Cov}_{\vgamma})}+tr(\sum_{i=1}^{N}\frac{n_{T}(i)}{\sigma_{2}^{2}+\sigma_{1}^{2}n_{T}(i)}\vx_{i}\vx_{i}^{T})\\
    &=\frac{1}{\lambda_{i}({\Cov}_{\vgamma})}+\sum_{i=1}^{N}\frac{n_{T}(i)}{\sigma_{2}^{2}+\sigma_{1}^{2}n_{T}(i)}tr(\vx_{i}^{T}\vx_{i})\\
    &\leq \frac{1}{\lambda_{i}({\Cov}_{\vgamma})}+\sum_{i=1}^{N}\frac{n_{T}(i)}{\sigma_{2}^{2}+\sigma_{1}^{2}n_{T}(i)}\\
    &\leq \frac{1}{\lambda_{i}({\Cov}_{\vgamma})}+\frac{N}{\sigma_{2}^{2}/T+\sigma_{1}^{2}}.
\end{align*}
The first inequality follows the Weyl's inequality. The second equality first uses linearity of trace, and then uses the cyclic property of trace. By assumption 1, we have $tr(\vx_{i}^{T}\vx_{i})=\|\vx_{i}\|_{2}^{2}\leq 1$, and the second last inequality holds.

Now we derive the mutual information of $\theta_{k}$ for each item $k \in [N]$, given the history and the meta-parameter $\vgamma$.
\begin{align*}
    I(\theta_{k};\mH_{T+1}|\vgamma)&=h(\theta_{k}|\vgamma)-h(\theta_{k}|\vgamma,\mH_{T+1})\\
    &=h(\theta_{k}|\vgamma)-\Mean[h_{T+1}(\theta_{k}|\vgamma)]\\
    &=\frac{1}{2}log(det(2 \pi e \sigma_{1}^{2}))-\Mean[\frac{1}{2}log(det(2 \pi e \hat{\sigma}_{T+1,\vgamma}^{2}(k))]\\
    &=\frac{1}{2}\Mean\Big[log[\sigma_{1}^{2}(\sigma_{1}^{-2}+\sigma_{2}^{-2}n_{T}(k))]\Big]\\
    &=\frac{1}{2}\mathbb{P}(n_{T}(k)\geq 1)\Mean\Big[log[1+\frac{\sigma_{1}^{2}}{\sigma_{2}^{2}}n_{T}(k)]\Big]\\
    &\leq \frac{1}{2}log(1+\frac{\sigma_{1}^{2}}{\sigma_{2}^{2}}\Mean[n_{T}(k)])\\
    &\leq \frac{1}{2}log(1+\frac{\sigma_{1}^{2}}{\sigma_{2}^{2}}T).
\end{align*}
The first inequality first uses the fact that $\mathbb{P}(n_{T}(k)\geq 1) \leq 1$, then follows the Jensen's inequality of log.
\end{proof}

\subsection{Proof for Lemma \ref{decomp for MI}}

\begin{lemma} \label{decomp for MI}
For any $k\in [N]$ and $\mH_{t}$-adapted sequence of actions $(A_{l})_{l=1}^{t-1}$, the following statements hold
\begin{align*}
    I(\theta_{k};\mH_{T+1}|\vgamma) &= \Mean\sum_{t} \mathbb{P}_{t}(k\in A_{t})I_{t}(\theta_{k};k,Y_{k,t}|\vgamma), \\
    K\sqrt{I_{t}(\vgamma;A_{t},\vY_{t})} &\geq \sum_{k \in [N]}\mathbb{P}_{t}(k \in A_{t})\sqrt{I_{t}(\vgamma;k,Y_{k,t})}.
\end{align*}
\end{lemma}

\begin{proof}
First, we derive the conditional mutual information of $\theta_{k}$ given history and the meta-parameter $\vgamma$. Note that, in the rounds when arm $k$ was not played, the mutual information gain for $\theta_{k}$ given $\vgamma$ is $zero$. In order words, $I_{t}(\theta_{k};k,Y_{k,t}|\vgamma)=0$ if arm $k$ was not played at round $t$. Then we used the chain rule of the mutual information ($I(X;Y,Z)=I(X;Z)+I(X;Y|Z)$) to finish the proof.
\begin{align*}
    I(\theta_{k};\mH_{T+1}|\vgamma) &= \sum_{t}I(\theta_{k};A_{t},\vY_{t}|\vgamma,\mH_{t})\\
    &=\Mean\sum_{t}I_{t}(\theta_{k};A_{t},\vY_{t}|\vgamma)\\
    &=\Mean\sum_{t}\sum_{a \in \mathcal{A}}\mathbb{P}_{t}(A_{t} = a)I_{t}(\theta_{k};a,\vY_{t}(a)|\vgamma)\\
    &=\Mean\sum_{t}\sum_{a \in \mathcal{A}}\mathbb{P}_{t}(A_{t} = a)\mathbbm{1}(k\in a)I_{t}(\theta_{k};k,Y_{k,t}|\vgamma)\\
    &+\Mean\sum_{t}\sum_{a \in \mathcal{A}}\mathbb{P}_{t}(A_{t} = a)I_{t}(\theta_{k};a\neg k,Y_{t}(a\neg k)|\vgamma,(k,Y_{k,t}))\\
    &=\Mean\sum_{t}\mathbb{P}_{t}(k \in A_{t})I_{t}(\theta_{k};k,Y_{k,t}|\vgamma).
\end{align*}
$a\neg k$ indicates that arm $k$ is removed from action set $a$, and $\vY_{t}(a)$ is the observed rewards of set $a$. The last inequality follows that, given $(\theta_{k};k,Y_{k,t}|\vgamma)$, history and $\vgamma$, $\theta_{k} \perp\!\!\!\!\perp (a\neg k,Y_{t}(a\neg k))$, and $\Mean\sum_{t}\sum_{a \in \mathcal{A}}\mathbb{P}_{t}(A_{t} = a)I_{t}(\theta_{k};a\neg k,Y_{t}(a\neg k)|\vgamma,(k,Y_{k,t}))=0$.

For the second part of the lemma, we use the fact that $I(X;Y,Z)\geq max(I(X;Z),I(X,Y))$, which is intuitive, as more observations will provide more mutual information gain. For a fixed $A_{t}$, we have
\begin{align*}
    \sum_{k \in [N]}\mathbb{P}_{t}(k \in A_{t})\sqrt{I_{t}(\vgamma;k,Y_{k,t})} &\leq \sum_{k \in [N]}\mathbb{P}_{t}(k \in A_{t})\sqrt{I_{t}(\vgamma;A_{t},\vY_{t})}\\
    &\leq K \sqrt{I_{t}(\vgamma;A_{t},\vY_{t})}.
\end{align*}
The second inequality is attained by noticing that $\sum_{k \in [N]}\mathbb{P}(k \in A_{t})\leq K$, as at most $K$ arms are played in each round. Note that this inequality typically show the benefits of information sharing among arms. Intuitively, with no feature sharing, we need to learn $N$ independent meta parameters separately, and we gain mutual information for each arm-specific meta parameter only when the corresponding arm is pulled. However, with feature sharing, we keep gaining information for $\vgamma$, which leads to a lower regret for learning meta parameter.   

\end{proof}

\section{Experiment details}\label{sec:appendix_expt}

\subsection{Robustness to model misspecification}\label{sec:appendix_additional_expt_robustness}
To facilitate scalablity, we assume that 
$\theta_i| \vx_i, \vgamma  \sim g(\theta_i|\vx_i, \vgamma)$. 
When the model $g$ is correctly specified, MTSS has shown superior theoretical and numerical performance. 
Intuitively, as this model is used to construct a prior for the feature-agnostic model, 
as long as the learned priors provide reasonable information compared to the manually specified ones, this framework is still valuable. 

In this section, we numerically study the impact of model misspecification on MTSS. 
When focusing on semi-bandits, the results under other problems are similar and therefore omitted. 
Specifically, 
instead of generating data according to $\theta_i \sim \normal\big(\vx_i^T \vgamma, \sigma_1^2\big)$, we consider the data generation process 
$\theta_i \sim \normal\big( \lambda cos(c_i\vx_i^T \vgamma)/c_i + (1 - \lambda) \vx_i^T \vgamma, \sigma_1^2\big)$, 
where $c$ is a normalization constant such that $c_i\vx_i^T \vgamma \in [-\pi/2,\pi/2]$, 
and $\lambda \in [0,1]$ controls the degree of misspecification. 
When $\lambda = 0$, we are considering the LMM; while when $\lambda = 1$, the features provide few  information through such a linear form. 

In results reported in Figure \ref{fig:simu_robust}, we observe that MTSS is fairly robust to model misspecifications. 
Although when $\lambda$ and $\sigma_1$ increase, the advantage over feature-agnostic decreases, MTSS  still always outperforms. 
Notably, MTSS always yield a nice sublinear regret unlike feature-determined TS, which further demonstrates the claimed robustness. 
We can even see that, perhaps surprisingly, when $\lambda = 1$, MTSS still outperforms feature-agnostic TS. 
This is mainly due to that, with an intercept term in $\vx_i$, our algorithm can at least learn $\prob(\theta_i | \vgamma)$ and enjoy the corresponding benefits. 
This is similar to the observations in \citet{kveton2021meta} and \citet{basu2021no}.

\renewcommand{\sharedwidthSubFig}{0.9}
\renewcommand{\sharedwidth}{0.7}
\renewcommand{\sharedheight}{3cm}
\begin{figure}[t]
     \centering
 \begin{subfigure}[b]{\sharedwidthSubFig\textwidth}
     \centering
     \includegraphics[width=\sharedwidth\textwidth, height = \sharedheight]{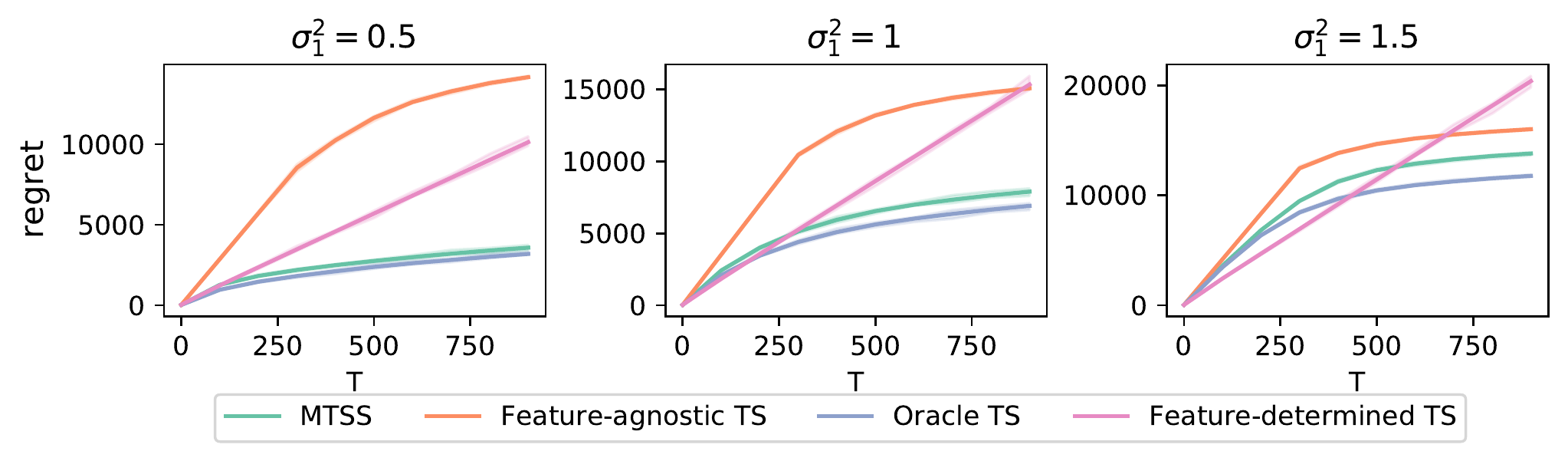}
     \caption{$\lambda$ = 0.25.}
 \end{subfigure}
 \begin{subfigure}[b]{\sharedwidthSubFig\textwidth}
     \centering
     \includegraphics[width=\sharedwidth\textwidth, height = \sharedheight]{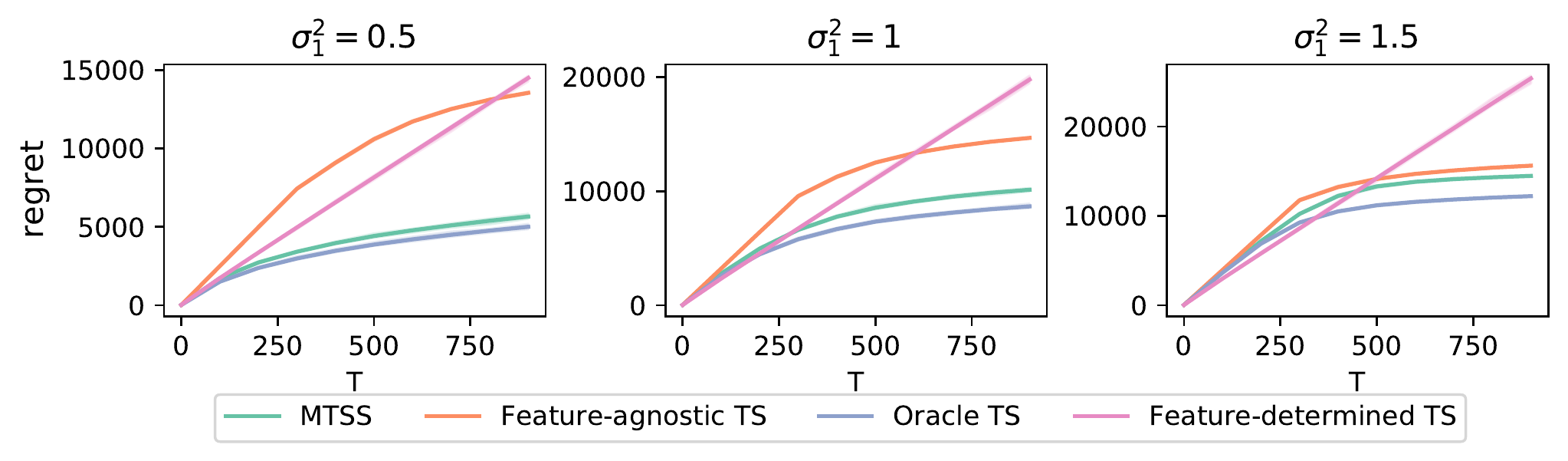}
     \caption{$\lambda$ = 0.5.}
 \end{subfigure}
 \begin{subfigure}[b]{\sharedwidthSubFig\textwidth}
     \centering
     \includegraphics[width=\sharedwidth\textwidth, height = \sharedheight]{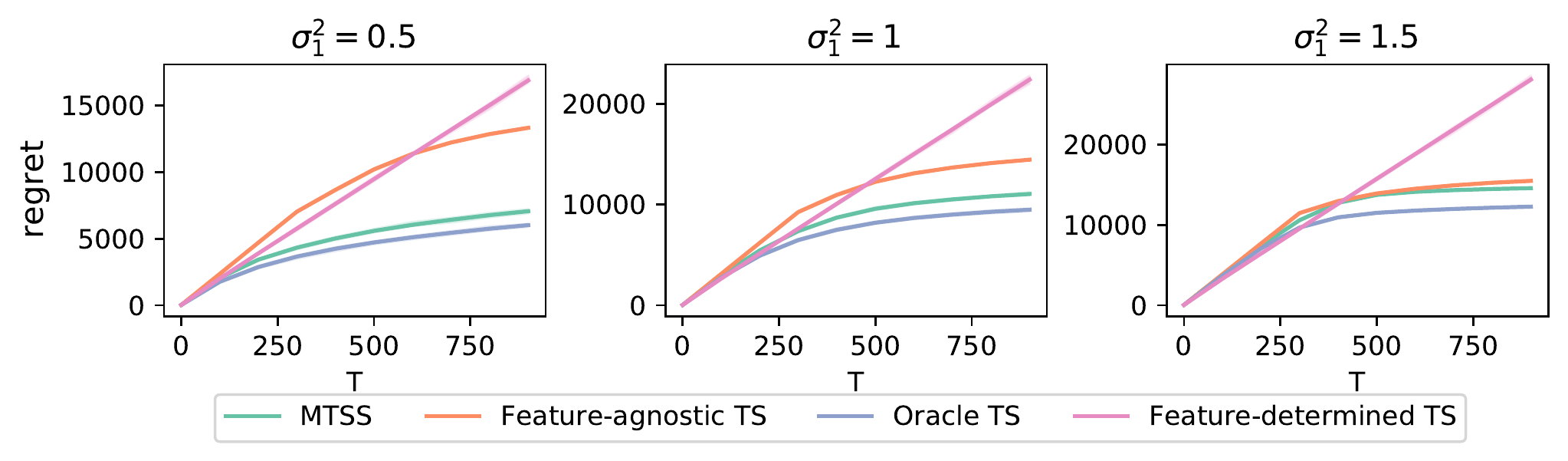}
     \caption{$\lambda$ = 0.75.}
 \end{subfigure}
 \begin{subfigure}[b]{\sharedwidthSubFig\textwidth}
     \centering
     \includegraphics[width=\sharedwidth\textwidth, height = \sharedheight]{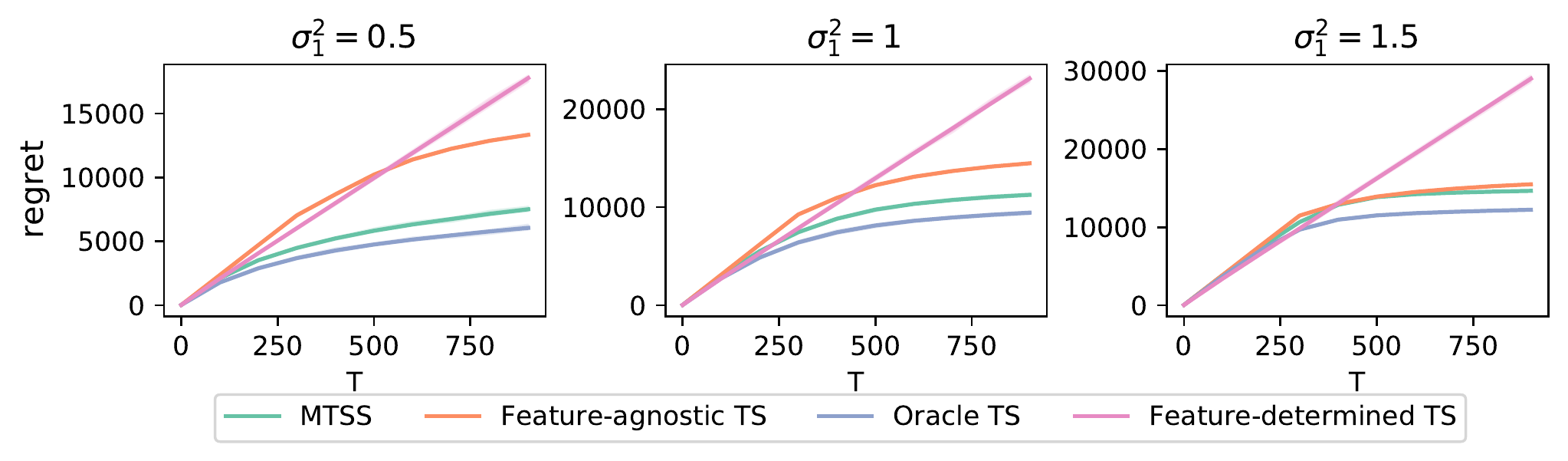}
     \caption{$\lambda$ = 1.0.}
 \end{subfigure}
\caption{
Robustness results. Shared areas indicate the standard errors of the averages.}
\label{fig:simu_robust}
\end{figure}


\subsection{Experiment results with cold-start problems}\label{sec:appendix_additional_expt_cold_start}
\begin{figure*}[t]
 \centering
 \includegraphics[width=0.9\textwidth, height = 4cm]{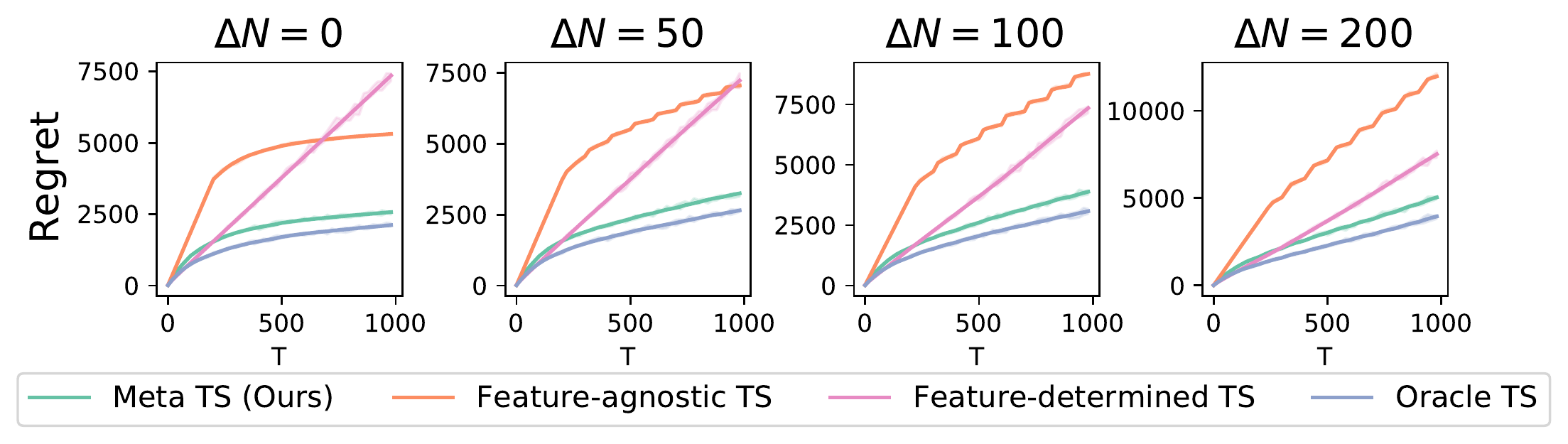} 
 \vspace{-0.3cm}
\caption{
Experiment results for semi-bandits with the cold start problem.}
\label{fig:cold}
\end{figure*}

In real-world applications, the set of items is typically not fixed. 
New items will be frequently introduced, and old items will be removed. 
Since there is no logged data for those newly-added items, 
such a challenge is typically referred to as the \textit{cold-start} problem. 

In this section, we compare the performance of various methods with the existence of the \textit{cold-start} problem. 
We use semi-bandits as an example. 
Specifically, we set $L = 1000, T= 1000, K = 5, p = 5, \sigma_1 = \sigma_2 = 1$. 
We start with $N = 1000$ items. 
The main difference with the experiments in the main text is that, every $100$ time points, we will remove $\Delta N$ existing items and introduce $\Delta N$ new items. 
We vary the value of $\Delta N$ from $0$ (no cold-start problem) to $200$. 

The experiment results can be found in Figure \ref{fig:cold}. 
As expected, in such a changing environment, all algorithms suffer a linear regret. 
The performance of feature-agnostic TS deteriorates significantly, as no information can be carried over to the new items.  
The difference between the regret of oracle-TS and MTSS is fairly stable, which implies that MTSS has learned the generalization function well and performs almost the same as oracle-TS eventually. 
MTSS consistently outperforms feature-agnostic TS and feature-determined TS. 

\subsection{Additional experiment results under other hyperparameter settins}\label{sec:appendix_additional_expt}

In this section, we present more simulation results under other combinations of $L, K, d$. 
See Figure \ref{fig:simu_more} for details. 
Overall, the performance and conclusions are fairly consistent with the ones presented in the main text. 

\renewcommand{\sharedwidthSubFig}{0.9}
\renewcommand{\sharedwidth}{0.7}
\renewcommand{\sharedheight}{3cm}

\begin{figure}[h]
     \centering
 \begin{subfigure}[b]{\sharedwidthSubFig\textwidth}
     \centering
     \includegraphics[width=\sharedwidth\textwidth, height = \sharedheight]{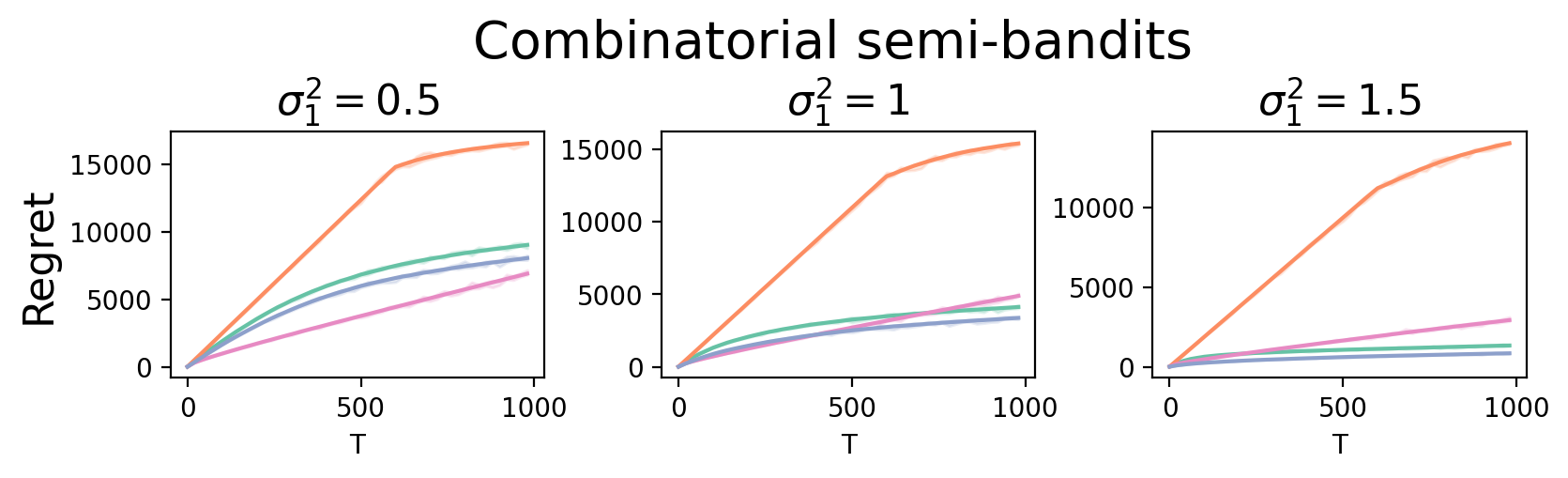}
     \caption{Combinatorial semi-bandits:$L = 1000, K = 5, p = 10$}
 \end{subfigure}
 \begin{subfigure}[b]{\sharedwidthSubFig\textwidth}
     \centering
     \includegraphics[width=\sharedwidth\textwidth, height = \sharedheight]{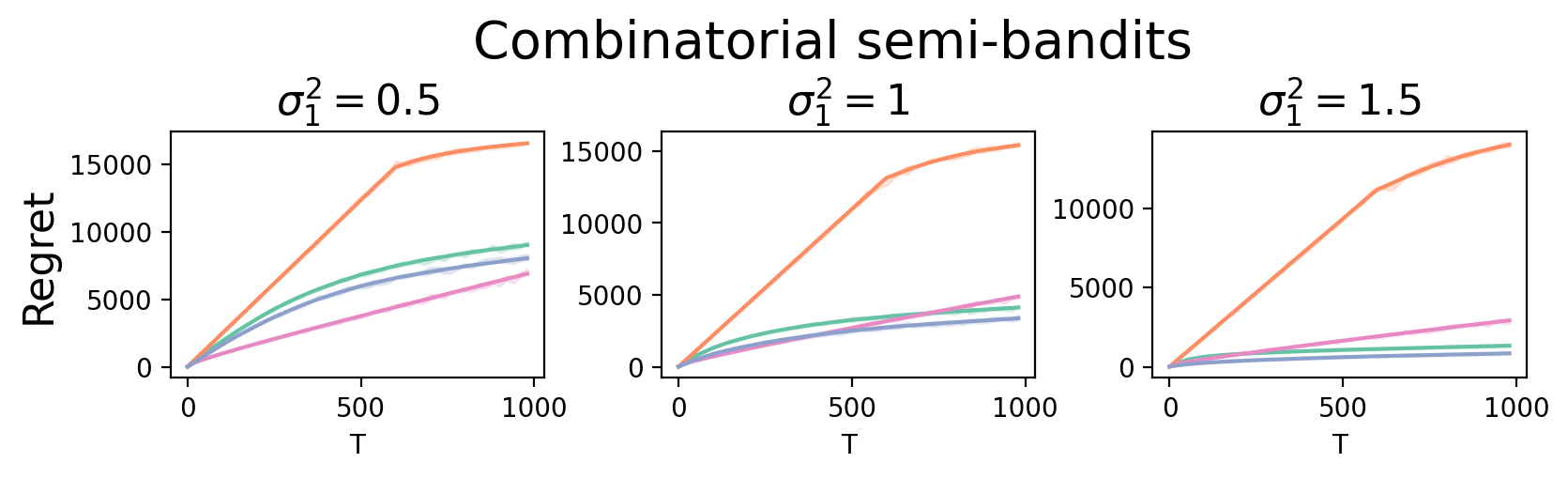}
     \caption{Combinatorial semi-bandits:$L = 3000, K = 5, p = 10$}
 \end{subfigure}
 \begin{subfigure}[b]{\sharedwidthSubFig\textwidth}
     \centering
     \includegraphics[width=\sharedwidth\textwidth, height = \sharedheight]{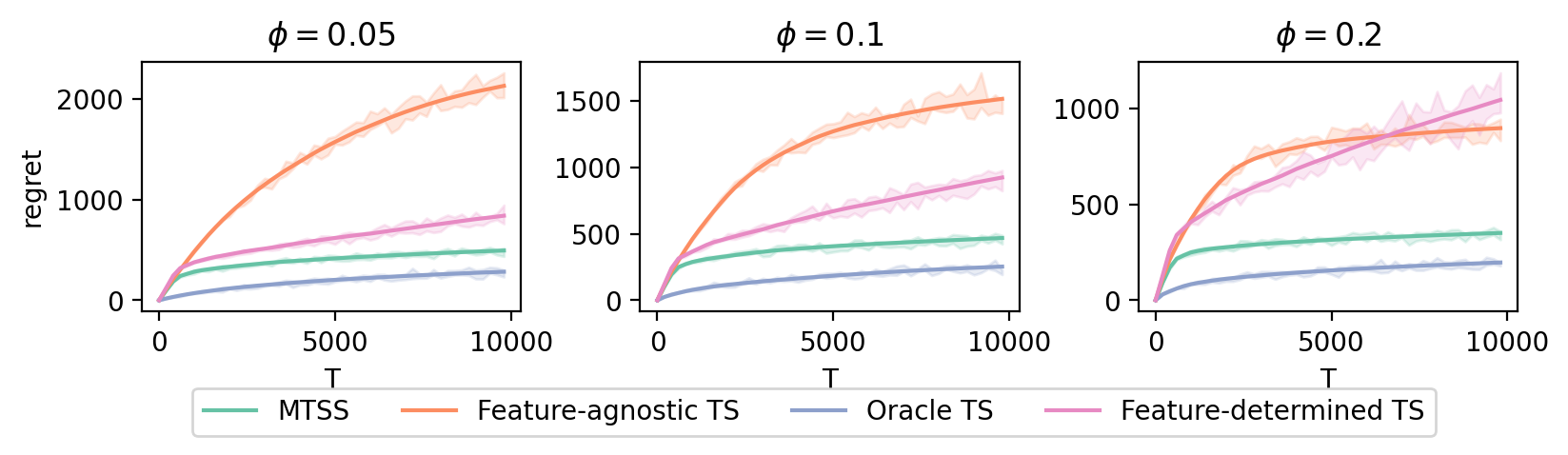}
     \caption{Cascading bandits:$L = 250, K = 5, p = 3$}
 \end{subfigure}
 \begin{subfigure}[b]{\sharedwidthSubFig\textwidth}
     \centering
     \includegraphics[width=\sharedwidth\textwidth, height = \sharedheight]{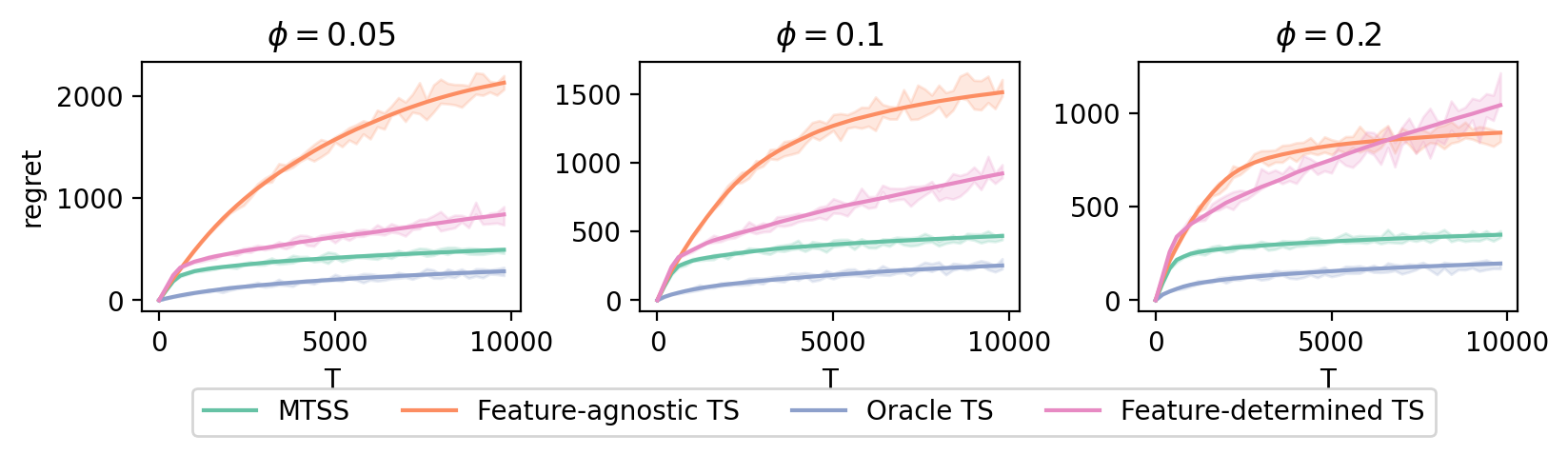}
     \caption{Cascading bandits:$L = 1000, K = 5, p = 10$}
 \end{subfigure}
 \begin{subfigure}[b]{\sharedwidthSubFig\textwidth}
     \centering
     \includegraphics[width=\sharedwidth\textwidth, height = \sharedheight]{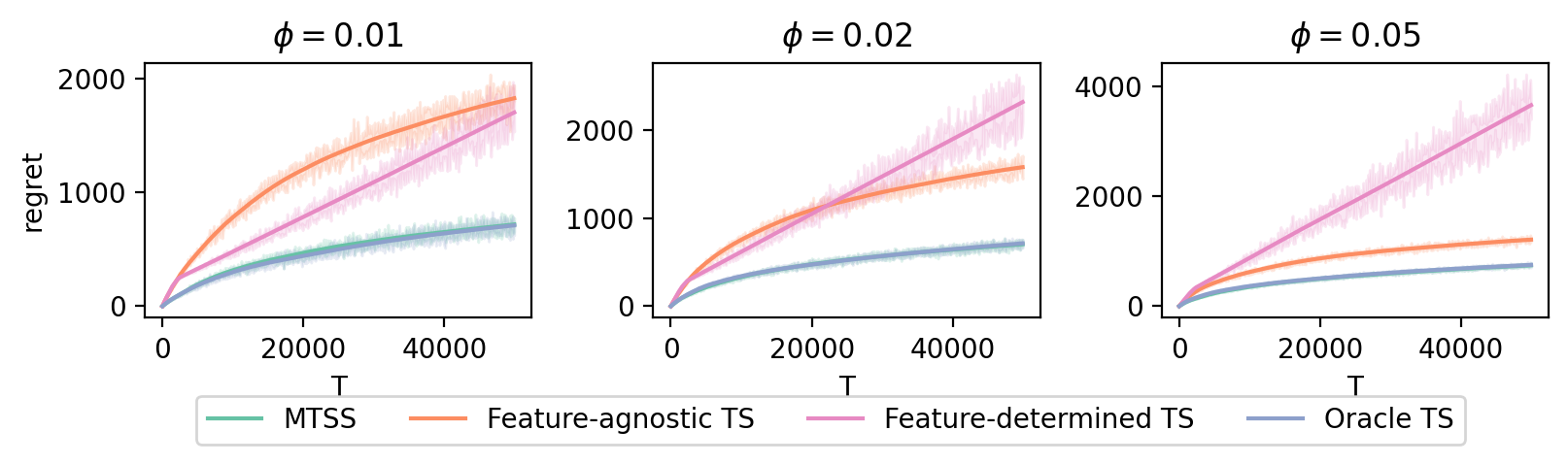}
     \caption{MNL bandits: $L = 1000, K = 10, p = 3$}
 \end{subfigure}
 \begin{subfigure}[b]{\sharedwidthSubFig\textwidth}
     \centering
     \includegraphics[width=\sharedwidth\textwidth, height = \sharedheight]{Fig/more/Final_MNL_L_1000_K_10_p_3.png}
     \caption{MNL bandits: $L = 1000, K = 10, p = 3$}
 \end{subfigure}
\caption{
Simulation results under additional settings.}
\label{fig:simu_more}
\end{figure}

\subsection{Additional experiment details}\label{sec:appendix_additional_expt_details}

In this section, we first introduce how we evaluate the performance of the learning algorithms and the low-rank matrix factorization, which is widely used to construct features. Then, details for each real experiment are discussed.

\textbf{Evaluation of Learning Algorithm.} While the synthetic experiments compare the learning algorithms by Bayes Regret defined in the main context, here for the real experiment, we focus on the expected cumulative regret conditioned on the true $\vthe$, which is derived carefully from the dataset. Mathematically,
 \begin{equation*}
      R(T, \vthe) = 
   \sum_{t=1}^T \big[ 
   \max_{a \in \mA} r(a, \vthe) - r(A_t, \vthe)
   \big]. 
 \end{equation*}

 \textbf{Low-rank Matrix Factorization.}
 Motivated by the collaborative filtering approach in recommender systems, low-rank factorization is widely used to construct the vectors of features. Suppose $A \in \mathcal{R}^{U \times N}$ includes the U observations of N items. Let $A \approx U\Sigma V^{T}$ be a rank-p truncated SVD of $A$, where $U\in\mathcal{R}^{U \times p}$, $\Sigma \in\mathcal{R}^{p \times p}$, and $V \in \mathcal{R}^{N \times p}$. Then the features of items are the rows of $V\Sigma$.  
 
 \subsubsection{Cascading Bandits}
 Here, we use the data related to business and reviews from the Yelp Dataset Challenge. For our experiments, we extract $N = 3000$ restaurants with the most number of reviews and $U = 20K$ users writing the most number of reviews. Similar to \citet{zong2016cascading}, we aim to maximize the probability of the user being attracted to at least one restaurant recommended. Following the experiment in \citet{zong2016cascading}, we convert the review data to an observation matrix $W\in \{0,1\}^{U\times N}$, where each entry indicates if the user is attracted by the restaurant, by assuming that a restaurant will attract a user if the user reviewed the restaurants at least once before. After that, we split $W$ into two distinct parts $W_{train}\in \{0,1\}^{\frac{U}{2}\times N}$ and $W_{test}\in \{0,1\}^{\frac{U}{2}\times N}$. While the $W_{train}$ is used to extract the features of each restaurant, the $W_{test}$ is used to evaluate the learning algorithms. Specifically, we applied the low-rank matrix factorization to $W_{train}$ to derive the features of restaurants with $p = 10$. The final features are standardized in the experiment, and an intercept is considered, which leads to $d = 11$. Finally, the true parameter $\vthe$ is computed by taking the sample average of $W_{test}$, and the true parameter 
 $\phi$ is derived appropriately from the $W_{test}$ by analyzing its posterior distribution. For each round, the observation is randomly selected from $W_{test}$. 
 
 \subsubsection{Semi-Bandits}
 Following the experiment setup in \citet{wen2015efficient}, we use the Adult dataset, which includes features of $33K$ people. In our experiment, we focus on only $N=3000$ people randomly selected. Our objective is to identify a set of $K=20$ users among the $3000$ people, including ten females and ten males, that are most likely to accept an advertisement. We considered $d=4$ features including age, gender, whether the person works more than $40$ hrs per week, and the length of education in years. Finally, we compute the true parameters from the dataset appropriately. First, we assume that the true expected acceptance probability (i.e., $\vthe$) depends on the user's income class. Specifically, 
 \begin{align*}
 \theta_{i}=\begin{cases}
    .15, & \text{income $>$ 50K}.\\
    .05, & \text{otherwise}.
  \end{cases}
 \end{align*}
 Then, the true parameter 
 $\sigma_{1}$ is learned by investigating the corresponding posterior distribution.
 
 \subsubsection{MNL Bandits}
 Following the experimental setup in \citet{oh2019thompson}, we use the dataset ``MovieLens 1M'' for our experiment. The dataset includes $1$ million ratings of $6K$ movies from $U = 4K$ users. In our experiment, we use $N=1000$ movies with the most ratings. While the range of ratings is from $1$ to $5$, we divide the ratings by $5$ and consider the utility of a movie to a user. Let the rating matrix be $W \in \mathcal{R}^{U\times N}$. We split $W$ into equal-size training dataset $W_{train}\in \mathcal{R}^{\frac{U}{2}\times N}$, and test dataset $W_{test}\in \mathcal{R}^{\frac{U}{2}\times N}$. Since most ratings are not complete, as most users do not review all the selected movies, we first implement the low-rank matrix completion \cite{keshavan2009low} to fill the missing ratings in $W_{train}$. Similar to \citet{oh2019thompson}, we then apply the low-rank matrix factorization to the imputed $W_{train}$ to construct the feature vector of each movie with $p = 5$. Then, we use the $L2$ normalization technique to normalize the features. Also, we consider including an intercept in the model. Therefore, $d = 6$ in the experiment. After that, we get the true mean utility $v_{i}$ of each movie as the sample average of $W_{test}$, and the true parameter $\vthe$ is obtained directly. Finally, we learn the true parameter 
 $\phi$ from $W_{test}$ in the same way as before.

\end{document}